\documentclass[12pt,journal,onecolumn]{IEEEtran}

\usepackage{amsthm}
\usepackage{amsmath}
\usepackage{multirow}
\usepackage{color}
\usepackage{booktabs}
\usepackage[all]{xy}
\usepackage{psfrag}
\usepackage{graphicx}
\usepackage{tikz}
\usepackage{bm}
\usepackage{framed,multirow}
\usepackage{amssymb}
\usepackage{latexsym}
\usepackage{url}
\usepackage{xcolor}
\definecolor{newcolor}{rgb}{.8,.349,.1}

\newcommand\sgn[0]{\mathrm{sgn}}

\newcommand\R[0]{\mathcal{R}}
\newcommand\N[0]{\mathcal{N}}
\newcommand\Ap[0]{\mathcal{A}}
\newcommand\M[0]{\mathcal{M}}
\newcommand\A[0]{A}
\newcommand\Q[0]{Q}
\newcommand\PH[0]{\phi}
\newcommand\cc[0]{\bm{c}}

\newcommand\x[0]{\bm{x}}
\newcommand\y[0]{\bm{y}}

\newcommand\rr[0]{\bm{r}}
\newcommand\avector[0]{\bm{a}}

\newtheorem{theo}{Theorem}
\newcommand\fns{\footnotesize}

\newcommand{\pg}[1]{}

\newcommand{\rej}{\mathord{\includegraphics[height=1.6ex]{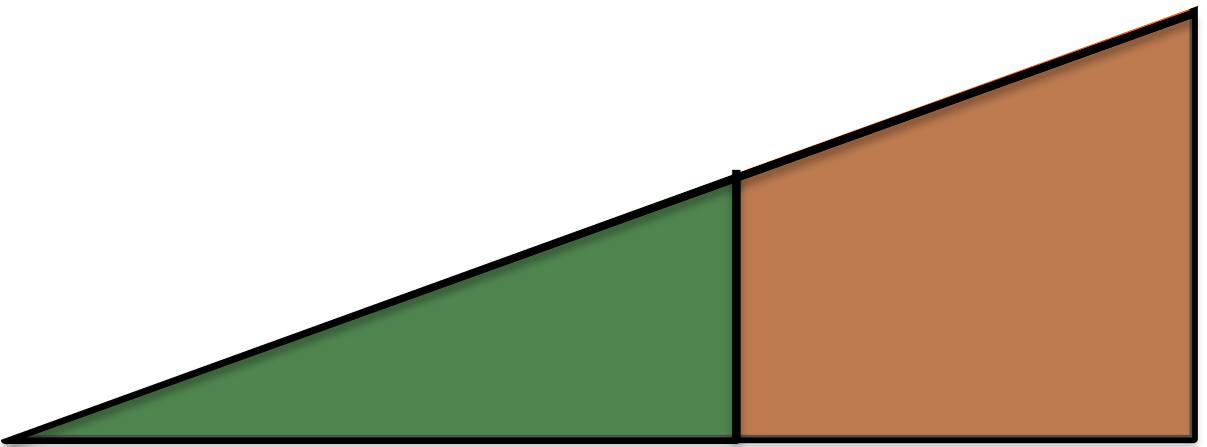}}}
\newcommand{\rc}{\mathord{\includegraphics[height=1.6ex]{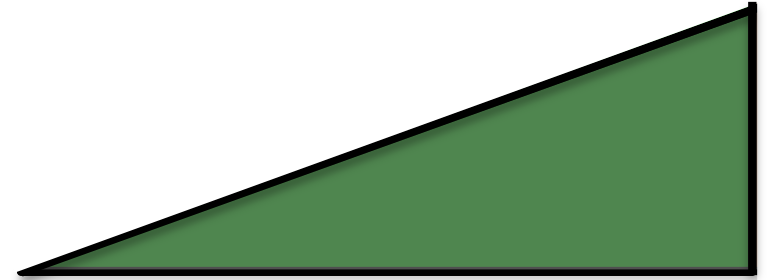}}}
\newcommand{\rnc}{\mathord{\includegraphics[height=1.6ex]{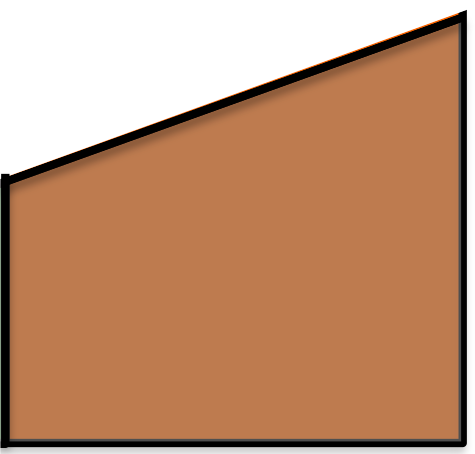}}}
\newcommand{\all}{\mathord{\includegraphics[height=3.2ex]{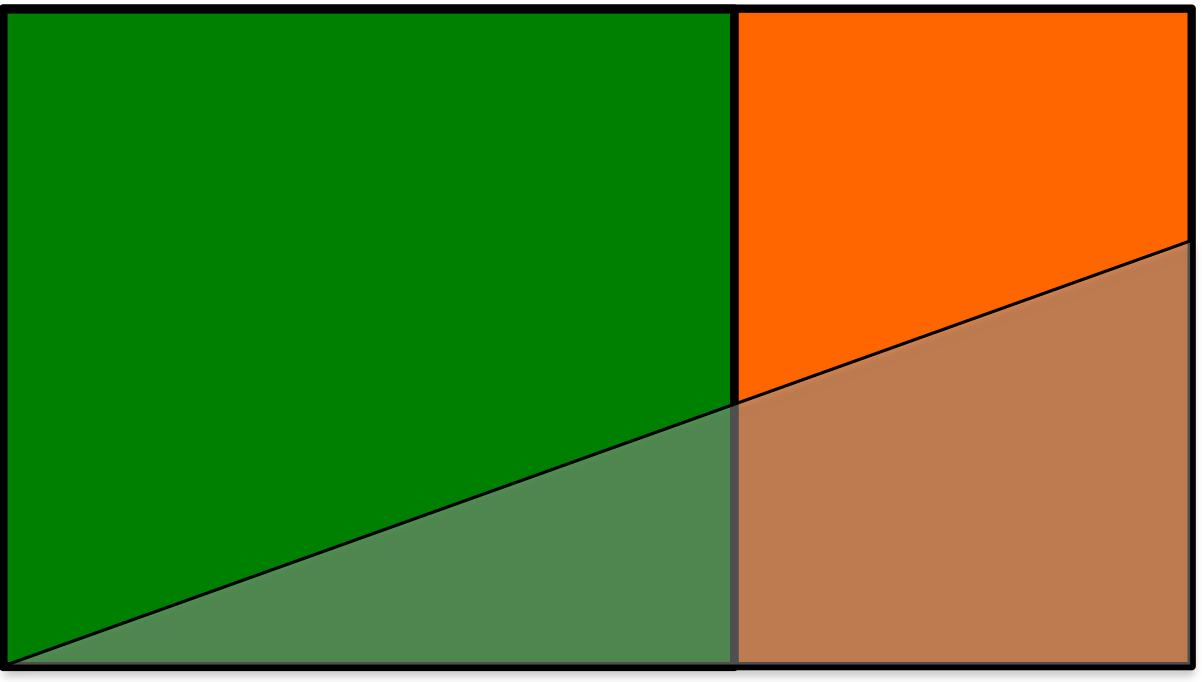}}}
\newcommand{\nrej}{\mathord{\includegraphics[height=3.2ex]{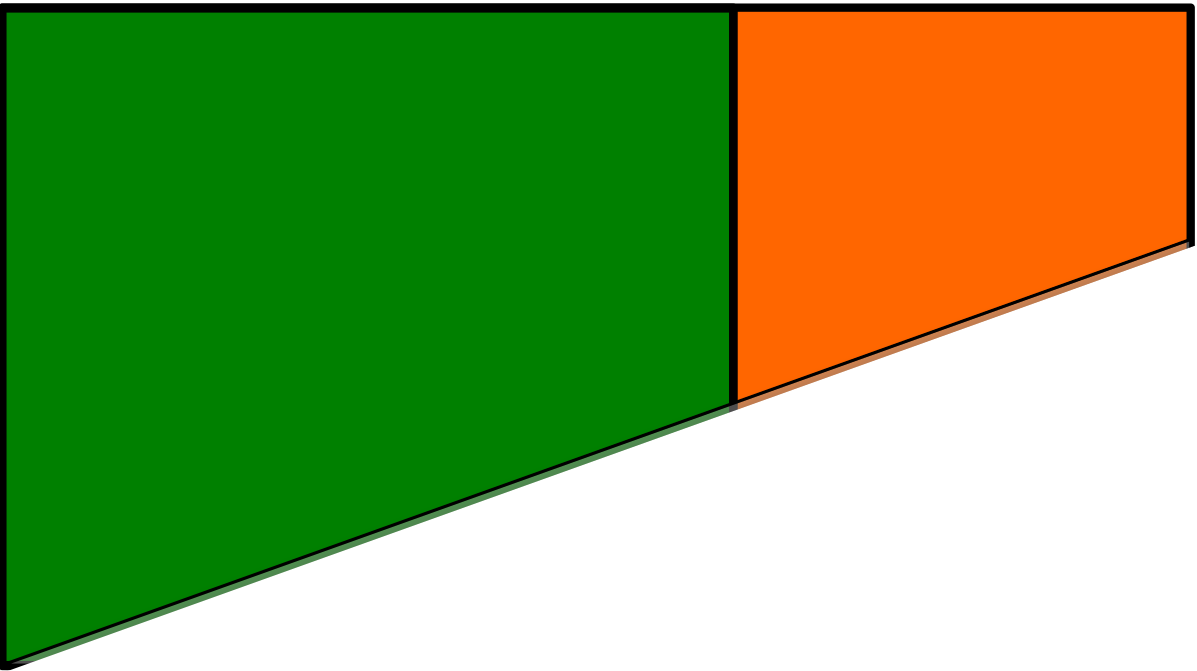}}}
\newcommand{\nrejc}{\mathord{\includegraphics[height=3.2ex]{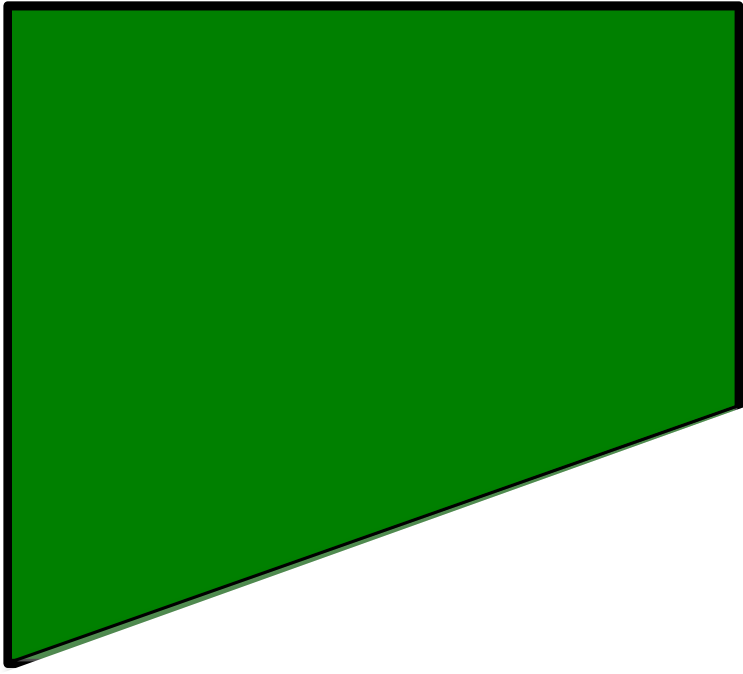}}}
\newcommand{\qtop}{\mathord{\includegraphics[height=3.2ex]{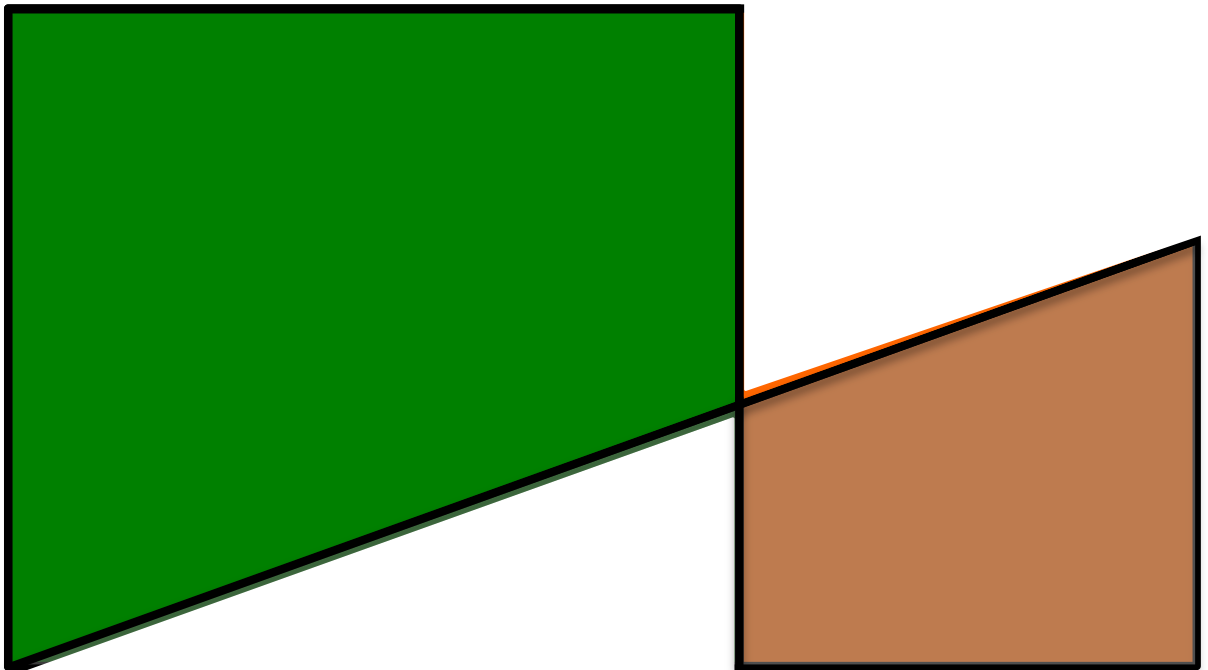}}}
\newcommand{\cor}{\mathord{\includegraphics[height=3.2ex]{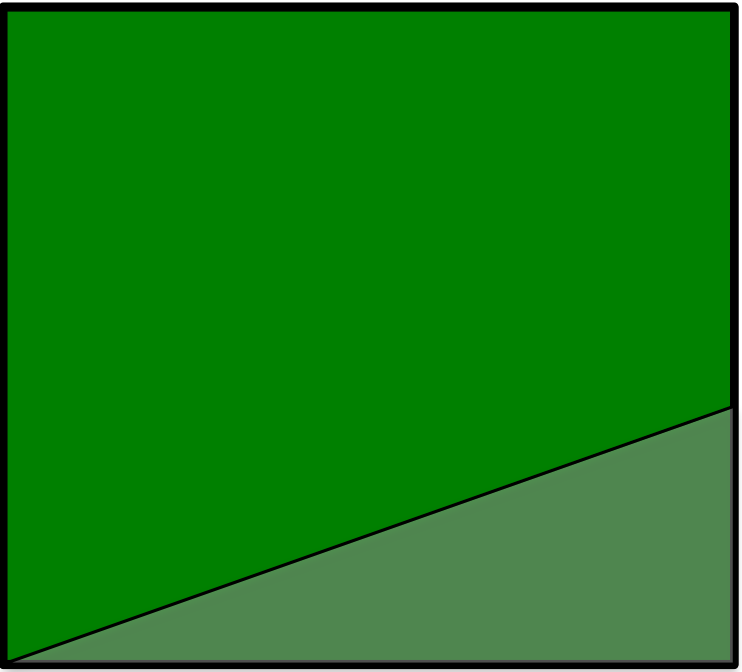}}}
\newcommand{\nc}{\mathord{\includegraphics[height=3.2ex]{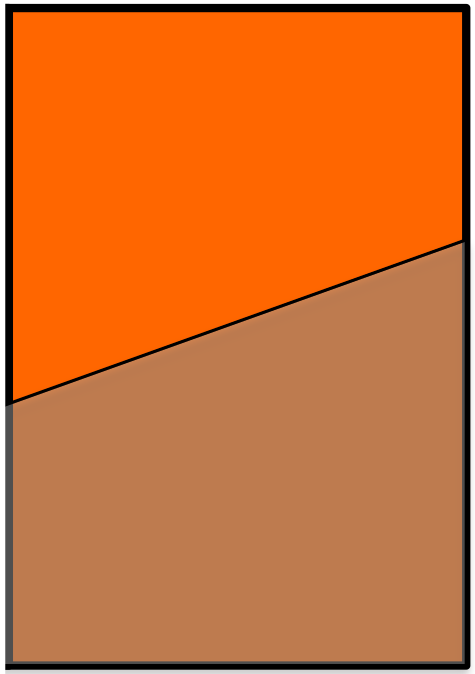}}}

\begin{document}
\title{Performance measures for classification systems with rejection}
\author{Filipe Condessa, Jos\'{e} Bioucas-Dias, and Jelena Kova\v{c}evi\'{c} 
\thanks{The authors gratefully acknowledge support from the NSF through award 1017278 and the CMU CIT Infrastructure Award. Work partially supported by grant  SFRH/BD/51632/2011, from Funda\c{c}\~{a}o para a Ci\^{e}ncia e Tecnologia and the CMU-Portugal (ICTI) program.. Filipe Condessa is with Instituto de Telecomunica\c{c}\~oes and the Dept. of Electrical and Computer Engineering  at Instituto Superior T\'ecnico, University of Lisbon, Portugal, Center for Bioimage Informatics and the Dept. of Electrical and Computer Engineering at Carnegie Mellon University, Pittsburgh, PA, condessa@cmu.edu.
 Jos\'{e} Bioucas-Dias is with Instituto de Telecomunica\c{c}\~oes and Dept. of Electrical and Computer Engineering at Instituto Superior T\'ecnico, University of Lisbon, Portugal, bioucas@lx.it.pt.
Jelena Kova\v{c}evi\'{c} is with the Dept. of Electrical and Computer Engineering, Center for Bioimage Informatics, and the Dept. of Biomedical Engineering at Carnegie Mellon University, Pittsburgh, PA, jelenak@cmu.edu.}
}
\maketitle

\begin{abstract}Classifiers with rejection are essential in real-world applications where misclassifications and their effects are critical.
However, if no problem specific cost function is defined, there are no established measures to assess the performance of such classifiers.
We introduce a set of desired properties for performance measures for classifiers with rejection, based on which we propose
a set of three performance measures for the evaluation of the performance of classifiers with rejection that satisfy the desired properties.
The \emph{nonrejected} accuracy measures the ability of the classifier to accurately classify nonrejected samples; the \emph{classification quality} measures the correct decision making of the classifier with rejector; and the \emph{rejection quality} measures the ability to concentrate all misclassified samples onto the set of rejected samples.
From the measures, we derive the concept of \emph{relative optimality} that allows us to connect the measures to a family of cost functions that take into account the trade-off between rejection and misclassification.
We illustrate the use of the proposed performance measures on classifiers with rejection applied to synthetic and real-world data.
\end{abstract}

\section{Introduction}
\label{sec:intro}
\pg{Classification with rejection as an interesting problem}
Classification with rejection is a viable option in real world applications of machine learning and pattern recognition, where the presence and cost of errors can be detrimental to performance.
This includes situations where the need to classify, in other words, when the cost of misclassifying is high (as in automated medical diagnosis~\cite{QuevedoRPL:12,CondessaBCOK:13}  or in landcover classification in remote sensing~\cite{CondessaBK:15a,CondessaBK:15c}), or where samples might be of no interest to the application (as in image retrieval \cite{VailayaFJZ:01}).
A classifier with rejection can also cope with unknown information, reducing the threat posed by the existence of unknown samples or mislabeled training samples that hamper the classifier's performance.
\pg{Inability to assess performance of classification with rejection}

\pg{Chow's design}
Classification with rejection was first analyzed in \cite{Chow:70}, where a rule for optimum error-reject trade-off was presented, Chow's rule.
In a binary classification setting, Chow's rule allows for the determination of a threshold for rejection such that the classification risk is minimized.
This requires both the knowledge of the \emph{a posterior} probabilities and the existence of a cost function that specifies the cost of misclassification and the cost of rejection.

\pg{Wegkamp and Fumera's design}
Multiple other designs for incorporating rejection into classification exist.
In a binary classification setting, the reject option can be embedded in the classifier.
An embedded reject option is possible through a risk minimization approach with the use of a hinge function, such as in \cite{ Wegkamp:07,BartlettW:08,YuanW:10}, to minimize classification risk.
It can also be achieved with support vector machines with embedded reject options, as described in~\cite{FumeraR:02,GrandvaletRKC:09,WegkampY:11}.
These embedded designs can also be extended to a rejection framework in nonbinary classification setting~\cite{PillaiFR:13}.

There is no standard measure for the assessment of the performance of a classifier with rejection.
Accuracy-rejection curves, used in~\cite{FumeraRG:00,FumeraR:02,FumeraR:04,SousaC:13}, and their variants based on the analysis of the $F_1$ score, used in ~\cite{FumeraPR:03,PillaiFR:13}, albeit popular in practical applications of classification with rejection have significant drawbacks.
Obtaining sufficient points for an accuracy rejection curve might not be feasible for classifiers with embedded reject option, which require retraining the classifier to achieve a different rejection ratio, or for classifiers that combine contextual information with rejection, such as~\cite{CondessaBCOK:13,CondessaBK:15c}.
This means that accuracy-rejection curves and the $F_1$ rejection curves, in the real world, are not able to describe the behavior of the classifier with rejection in all cases.

In~\cite{LandgrebeTPD:06}, a different approach is taken. A 3D ROC (receiver operating characteristic) plot of a 2D ROC surface is obtained by decomposing the false positive rate into false positive rate for outliers belonging to known classes and false positive rate for outliers belonging to unknown classes, with the VUC (volume under the curve) as the performance measure.
The use of ROC curves for the analysis of the performance suffers from the same problems associated with accuracy-rejection curves.

To fill this gap, we propose a set of desired properties for performance measures for classifiers with rejection, and a set of three performance measures that satisfy those properties.

A performance measure that evaluates the performance of a rejection
mechanism given a classifier should satisfy the following:
\begin{itemize}
\item Property I --- be a function of the fraction of rejected samples;
\item Property II --- be able to compare different rejection mechanisms working at the same fraction of rejected samples;
\item Property III --- be able to compare rejection mechanisms working at a different fractions of rejected samples when one rejection mechanism outperforms the other;
\item Property IV --- be maximum for a rejection mechanism that no other feasible rejection mechanism outperforms, and minimum for a rejection mechanism that all other feasible rejection mechanisms outperform.
\end{itemize}

These properties rely on being able to state whether one rejection mechanism qualitatively outperforms the other.
If a cost function exists that takes in account the cost of rejection and misclassification, the concept of outperformance is trivial, and this cost function not only satisfies the properties but is also the ideal performance measures for the problem in hand.
It might not be feasible, however, to design a cost function for each individual classification problem.
Thus, we derive a set of cases where the concept of outperformance is independent from a specific cost function (under the assumption that the cost of rejection is never greater than the cost of misclassification).

With the properties and the concept of outperformance in place, we present three measures that satisfy the above properties:
\begin{itemize}
\item \textbf{\emph{Nonrejected accuracy}} measures the ability of the classifier to accurately classify nonrejected samples;
\item \textbf{\emph{Classification quality}} measures the ability of the classifier with rejection to accurately classify nonrejected samples and to reject misclassified samples;
\item \textbf{\emph{Rejection quality}} measures the ability of the classifier with rejection to make errors on rejected samples only.
\end{itemize}

With the three measures in place, we can explore the best and worst case scenarios for each measure, for a given reference classifier with rejection.
We denote the proximity of a classifier with rejection to its best and worst case scenarios, with regard to a reference classifier with rejection, as \emph{relative optimality}.
This allows us to easily connect performance measures to problem specific cost functions.
For a classifier with rejection that rejects at two different numbers of rejected samples, the relative optimality defines the family of cost functions on which rejection at one number rejected samples is better, equal, or worse than rejection at the other number of rejected samples.

The rest of the paper is structured as follows.
In Section~\ref{sec:csr}, we present the classifier with rejection; we introduce the necessary notation in Section~\ref{subsec:notation}, define the three concepts of rejector outperformance that do not depend on cost functions in Section~\ref{subsec:comparative}, and present the desired performance measure properties in Section~\ref{subsec:properties}.
In Section~\ref{sec:measures}, we present the set of proposed performance measures.
In Section~\ref{sec:relative_optimality}, we connect the performance measures to cost functions by defining relative optimality.
In Section~\ref{sec:experimental}, we illustrate performance measures on real-world applications.
In Section~\ref{sec:conclusion}, we conclude the paper.

\section{Classifiers with rejection}
\label{sec:csr}
\subsection{Notation}
\label{subsec:notation}
A classifier with rejection can be seen as a coupling of a classifier  $C$ with a rejection system $R$.
The classification maps $n$ $d$-dimensional feature vectors $\x$ into $n$ labels $C(\x): \mathcal{R}^{d\times n} \rightarrow \{1, \hdots, K \} ^{n}$, such that 
\begin{equation*}
\hat{\y} = C(\x),
\end{equation*}
where $\hat{\y}$ denotes a labeling.
The rejector $R$ maps the classification (feature vectors and associated labels) into a binary rejection vector, $ R(\x,\hat{\y}):   \mathcal{R}^{d\times n} \times \{1, \hdots, K \} ^{n} \rightarrow  \{0, 1 \} ^{n}$, such that 
\begin{equation*}
\rr =R(\x,\hat{\y}),
\end{equation*} where $\rr$ denotes the binary rejection vector.
We define a classification with rejection $\hat{\y}^R$ as 
\begin{equation*}
\hat{\y}^R_i = \begin{cases}
\hat{\y}_i, \mbox{ if } \rr_i = 0, \\
0, \mbox{ if } \rr_i = 1,
\end{cases}
\end{equation*}
where $\rr_i$ corresponds to the binary decision to reject ($r_i = 1$) or not ($r_i = 0$) the $i$th classification, and $\hat{\y}^R_i=0$ denotes rejection.

By comparing the classification $\hat{\y}$ with its ground truth $\y$, we form a binary $n$-dimensional accuracy vector $\avector$, such that $\avector_i$ measures whether the $i$th sample is classified accurately.
The binary vector $\avector$ imposes a partition of the set of samples in two subsets $\Ap$ and $\M$, namely the subset of accurately classified samples and the subset of misclassified samples.
Let $\cc$ be a confidence vector associated with the classification $\hat{y}$, such that 
\begin{equation*}
\cc_i \geq \cc_j \implies \rr_i \leq \rr_j,
\end{equation*} this is, if sample $i$ is rejected, then all the samples $j$ with smaller confidence $\cc_i  <  \cc_j$ are also rejected.
We thus have the ground truth $\y$, the result of the classification $\hat{\y}$, and the result of the classification with rejection $\hat{\y}^R$.

Let $\hat{\cc}$ denote the reordering of the confidence vector $\cc$ in decreasing order.
If we keep $k$ samples with the highest confidence and reject the rest $n-k$ samples, we obtain two subsets:  $k$ nonrejected samples and $n-k$ rejected samples, $\mathcal{N}$ and $\mathcal{R}$\footnote{We note that $R$ corresponds to a rejector, a function that maps classification into a binary rejection vector, whereas $\mathcal{R}$ denotes a set of samples that are rejected.} respectively.
Our goal is to separate the accuracy vector $\avector$ into two subvectors
($\avector_{\bm \N}$ and $\avector_{\bm \R}$), based on the confidence vector $\cc$ such
that all misclassifications are in the $\avector_{\bm \R}$ subvector,
and all accurate classifications are in the $\avector_{\bm \N}$ subvector.
We should note that, since ${\bm \N}$ and ${\bm \R}$ have disjoint supports,
\begin{align}
\label{eq:fundamental_identity}
\|\avector\| = \|\avector_{\bm \N}\| + \|\avector_{\bm \R}\|,
\end{align}
for all ${\bm \N}, {\bm \R}$ such that ${\bm \N} \cap {\bm \R} = \emptyset$ and ${\bm \N } \cup {\bm \R} = \{ 1, \hdots, N\}$.
As we only work with the norm of binary vectors, we point that $\|\avector\|_0 = \|\avector\|_1$; for simplicity, we omit the subscript.

\begin{figure}[htp]
\begin{center}
\begin{tabular}{ccc}
\includegraphics[width=.125\columnwidth]{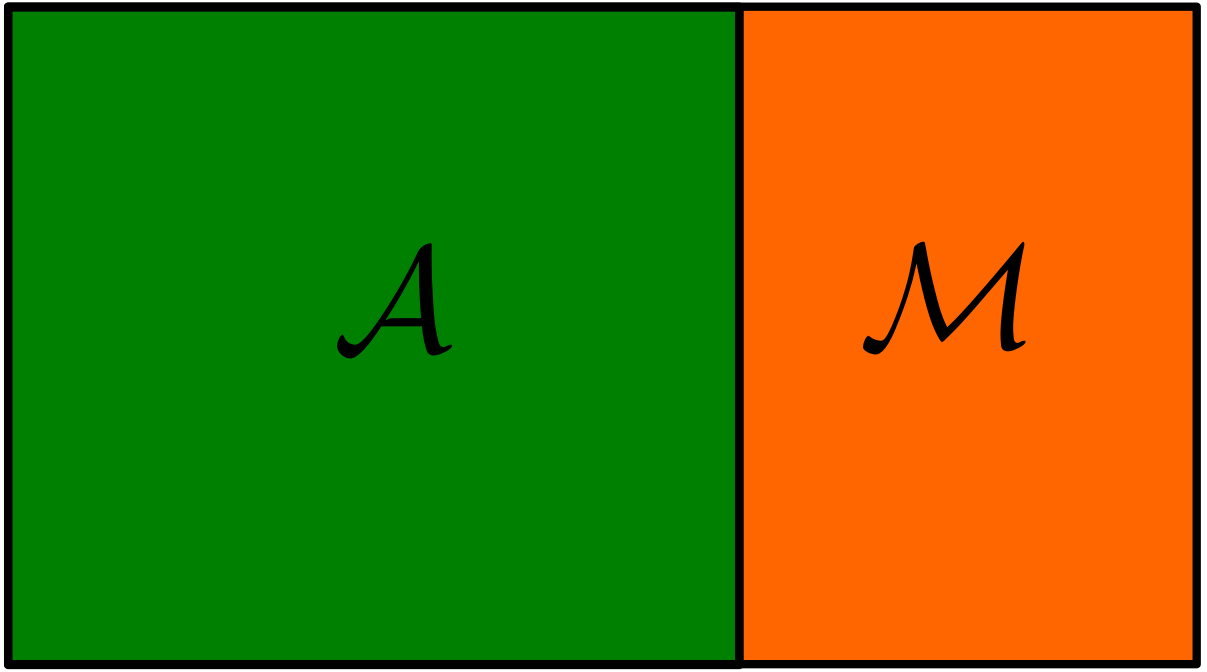}&
\includegraphics[width=.125\columnwidth]{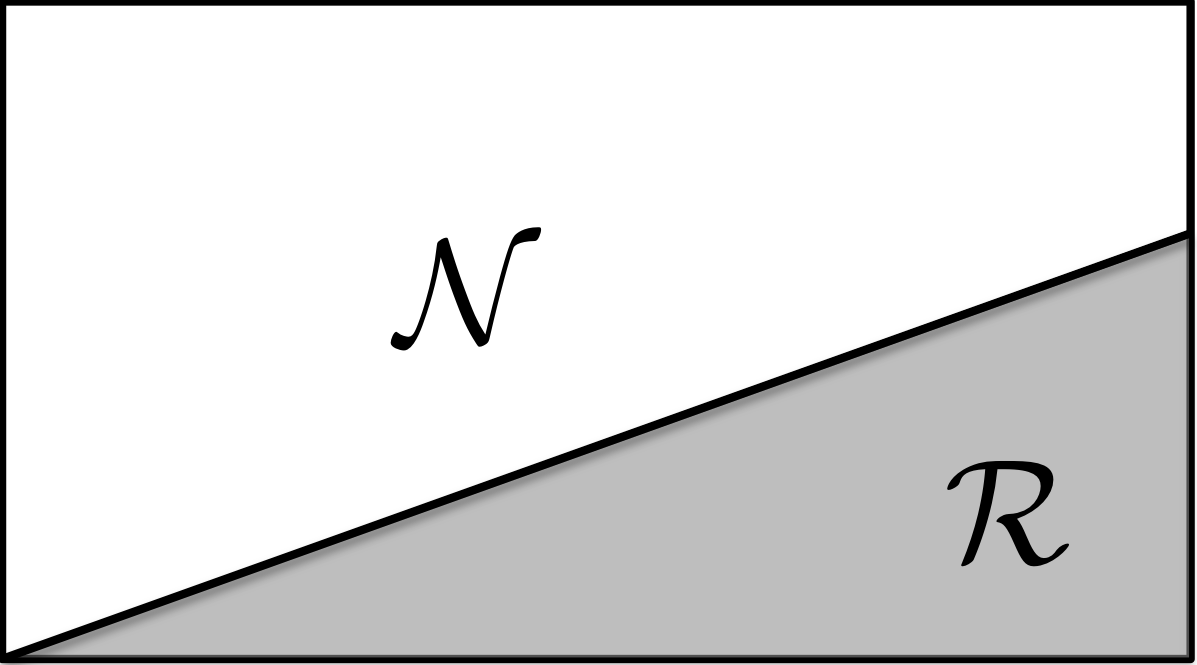}&
\includegraphics[width=.125\columnwidth]{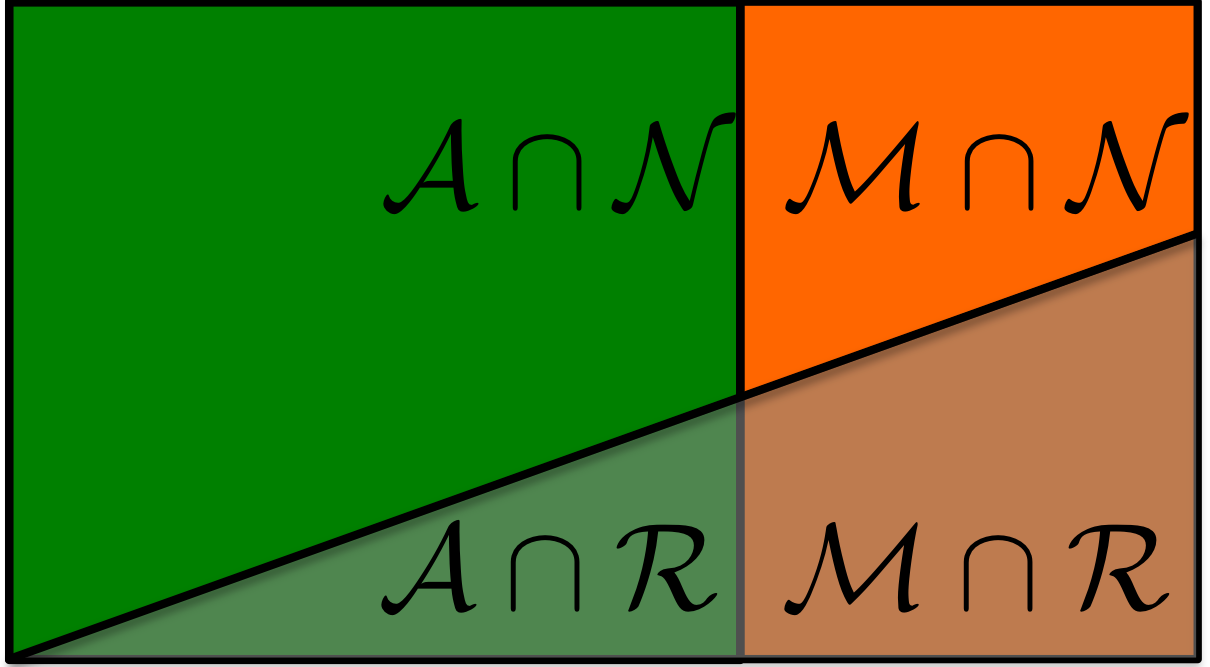}\\
\footnotesize ({a}) & 
\footnotesize ({b})& 
\footnotesize ({c}) \\
\end{tabular}
\caption{\label{fig:rej_example}
Partition of the sample space based on the performance of the (a) classification only (partition space $\Ap$ and $\M$); (b) rejection only (partition space $\R$ and $\N$); and (c) classification with rejection.
Green corresponds to accurately classified samples and orange to misclassified samples.
Gray corresponds to rejected samples and white to nonrejected samples.}
\end{center}
\end{figure} 
With the partitioning of the sample space into $\mathcal{A}$ and $\mathcal{M}$ according to the values of the binary vector $\avector$, and the partitioning of the sample space into $\bm \N $ and $\bm \R$, we thus partition the sample space as in Fig.~\ref{fig:rej_example}:
\begin{itemize}
\item $\mathcal{A} \cap \N$: samples {\bf a}ccurately classified and {\bf n}ot rejected; the number of such samples is $| \mathcal{A} \cap \N | =   \| \avector_\N\|$
\item $\mathcal{M} \cap \N$: samples {\bf m}isclassified and {\bf n}ot rejected; the number of such samples is $| \mathcal{M} \cap \N | =   \|{\mathbf 1} - \avector_\N\|$
\item $\mathcal{A} \cap \R$: samples {\bf a}ccurately classified and {\bf r}ejected; the number of such samples is $| \mathcal{A} \cap \R | =   \| \avector_\R\|$
\item $\mathcal{M} \cap \R$: samples {\bf m}isclassified and {\bf r}ejected; the number of such samples is $| \mathcal{M} \cap \R | =   \|{\mathbf{1}} - \avector_\R\|$
\end{itemize}

\subsection{Comparing classifiers with Rejection}
\label{subsec:comparative}
The comparison of the performance of two rejectors is nontrivial.
It depends on the existence of a problem specific cost function that takes in account the trade-off between misclassification and rejection.
If a cost function exists, the performance is linked to the comparison of the cost function evaluated on each rejector.
However, as previously stated, the design of a problem specific cost function might not be feasible.
Let $\rho$ denote the trade-off between rejection and misclassification, thus defining a family of cost functions where a misclassification has a unitary cost, a rejection has a cost of $\rho$ and an accurate classification has no cost.
Then, there are three general cases where it is possible to perform comparisons between the performance of two rejectors independently of $\rho$:
when the number of rejected samples is the same;
when the number of accurately classified samples not rejected is the same;
and when the number of misclassified samples not rejected is the same.
This is true for all $\rho$, if we assume that $0 \leq \rho \leq 1$, which is a reasonable assumption as $\rho < 0$ would lead to a rejection only problem (all samples rejected), and $\rho >1 $ would lead to a classification only problem (no samples are rejected).
Let $C$ denote a classifier with an accuracy vector $\avector$, and $R_1$ and $R_2$ denote two different rejection mechanisms that partition the sample space in $\mathcal{N}_{R_1},\mathcal{R}_{R_1}$ and $\mathcal{N}_{R_2},\mathcal{R}_{R_2}$ respectively.
\subsubsection*{Equal number of rejected samples}
If both rejectors reject the same number of samples, and if rejector $R_1$ has a larger number of accurately classified samples than $R_2$, then $R_1$ outperforms $R_2$.
\begin{figure}[hptb]
\begin{center}
\begin{tabular}{ccc}
\includegraphics[width=.125\columnwidth]{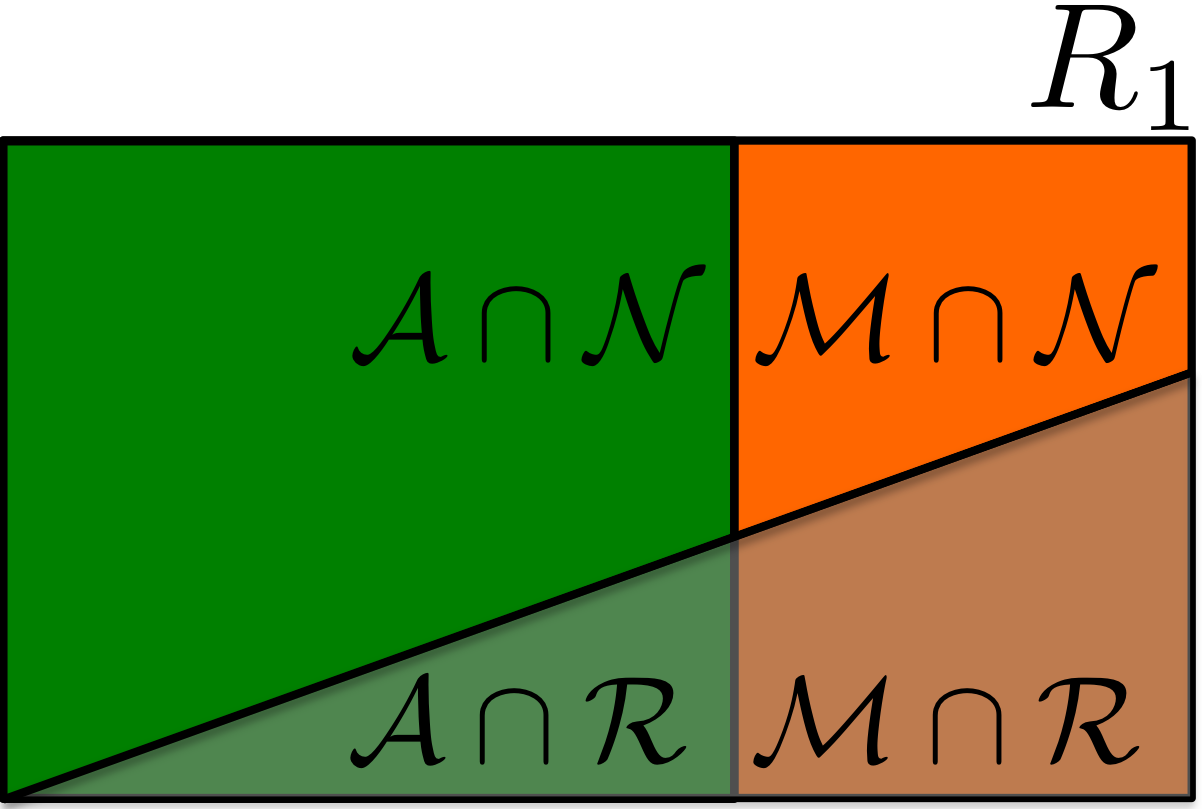}&
\raisebox{.5cm}{$\xrightarrow{\textrm{outperforms}}$}&\includegraphics[width=.125\columnwidth]{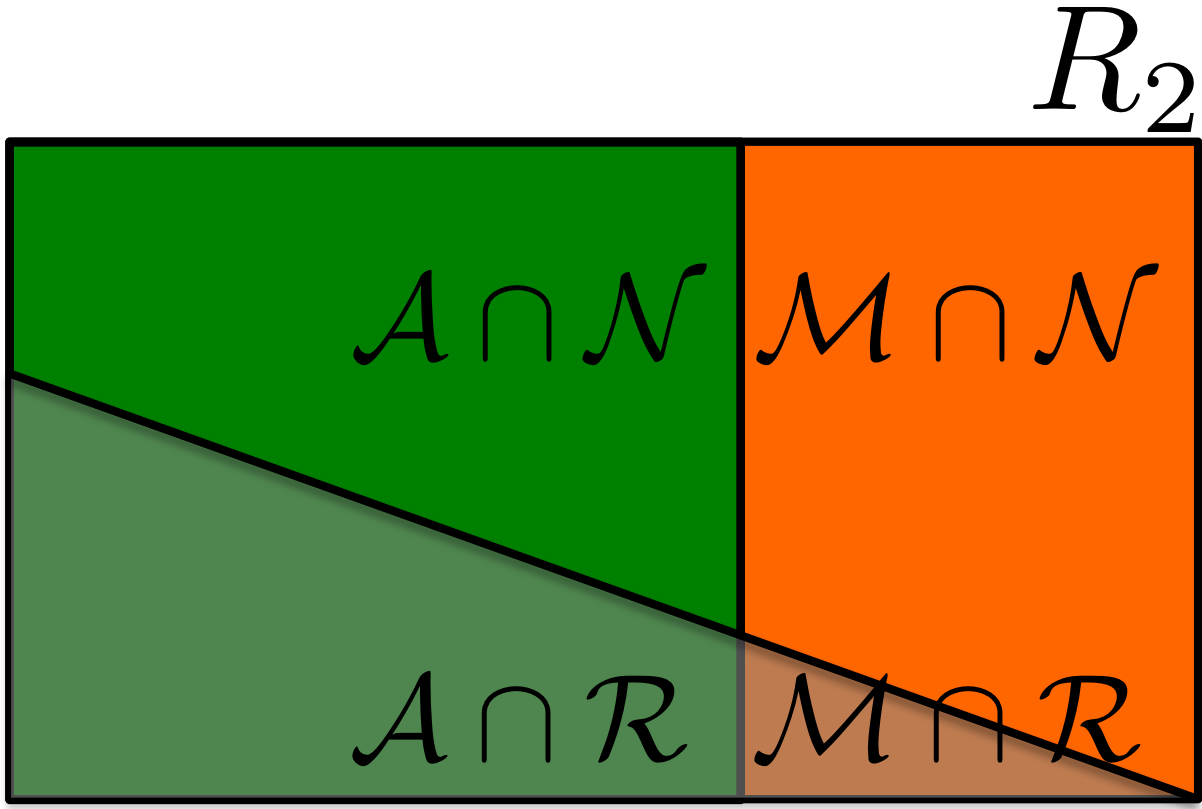}\\
\end{tabular}
\vspace{-0.5cm}
\end{center}
\end{figure}


\subsubsection*{Equal number of nonrejected accurately classified samples}
If both rejectors have the same number of accurately classified samples not rejected, and if rejector $R_1$ rejects more samples than $R_2$, then $R_1$ outperforms $R_2$.
\begin{figure}[hptb]
\begin{center}
\begin{tabular}{ccc}
\includegraphics[width=.125\columnwidth]{example/r1_case1.png}&
\raisebox{.5cm}{$\xrightarrow{\textrm{outperforms}}$}&
\includegraphics[width=.125\columnwidth]{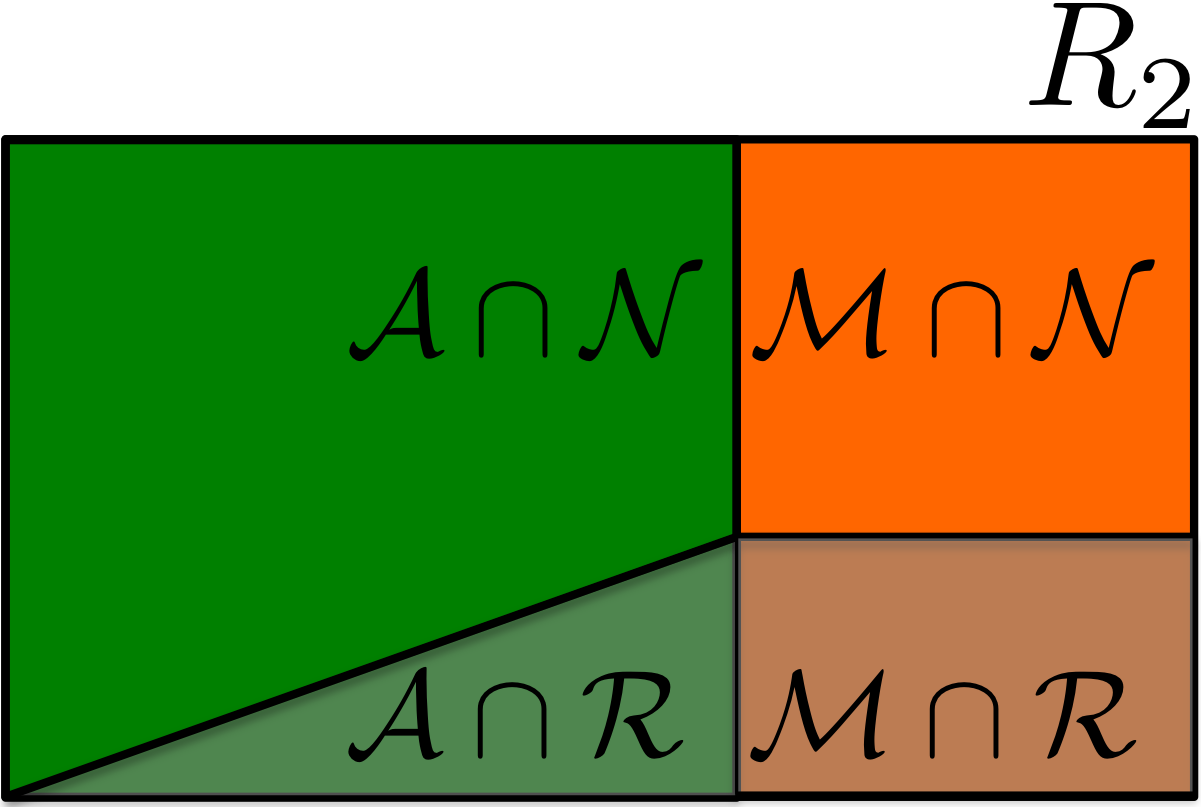}\\
\end{tabular}
\vspace{-0.5cm}
\end{center}
\end{figure} 


\subsubsection*{Equal number of nonrejected misclassified samples}
If both rejectors have the same number of misclassified samples not rejected, and if rejector $R_1$ rejects fewer samples than $R_2$, then $R_1$ outperforms $R_1$.
\begin{figure}[hptb]
\begin{center}
\begin{tabular}{ccc}
\includegraphics[width=.125\columnwidth]{example/r1_case1.png}&\raisebox{.5cm}{$\xrightarrow{\textrm{outperforms}}$}&
\includegraphics[width=.125\columnwidth]{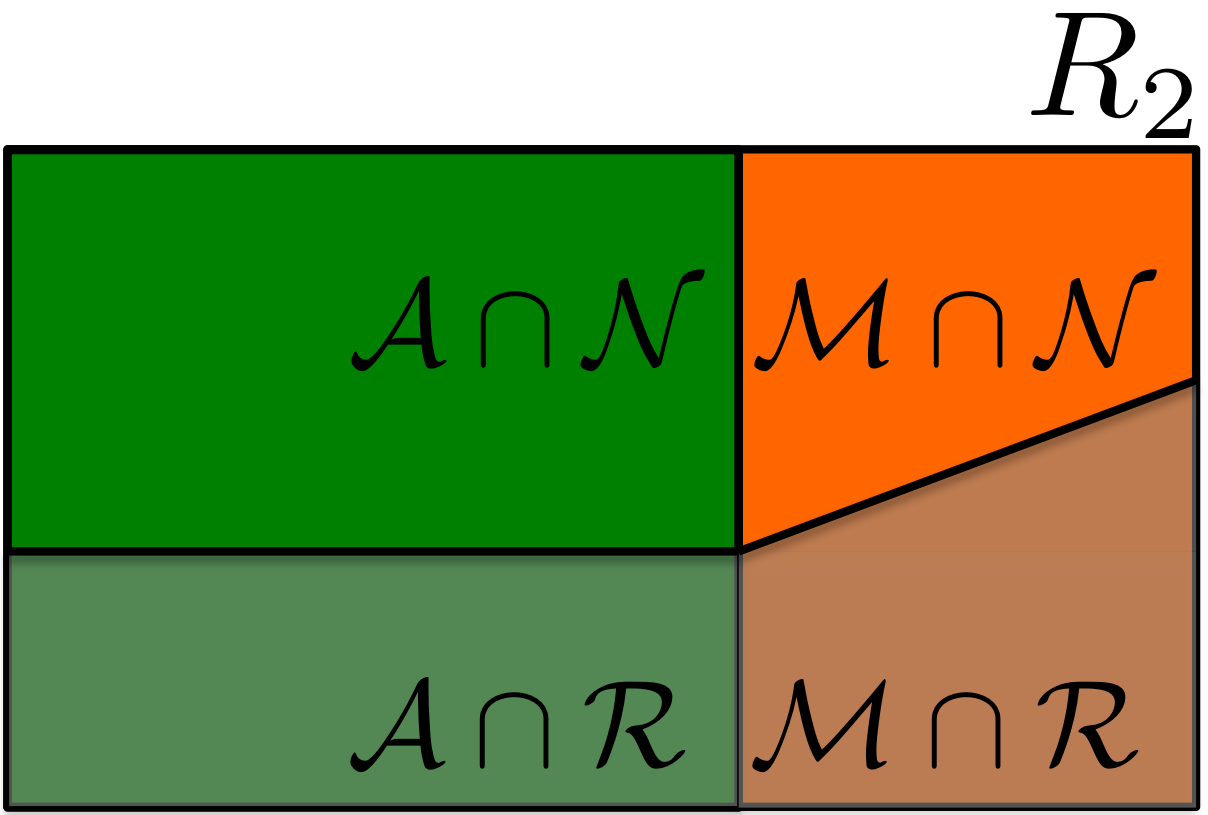}\\
\end{tabular}
\vspace{-0.5cm}
\end{center}
\end{figure} 


\subsection{Desired properties of performance measures}
\label{subsec:properties}
The definition of the rejection problem as the partition of the accuracy vector $\avector$ based on two disjoint supports $\N$ and $\R$ is general and allows us to define desired characteristics for any generic performance measure $\alpha$ that evaluates the performance of classification with rejection.

\pg{Rejected fraction}
We start by introducing the rejected fraction $r$, as the ratio of rejected samples versus the overall number of samples,
\begin{center}
\begin{equation}
\label{eq:rejrat}
r =  \frac{ n - k}{n} =  \frac{|\bm{ \R} |}{|\bm{ \R} | + |\bm {\N} |}  = \frac{\rej}{\all}.
\end{equation}
\end{center}
\subsubsection{Property I: Performance measure is a function of the rejected fraction}
\label{prop1}
The first desired characteristic of a performance measure $\alpha$, is for the measure $\alpha$ to be a function of number of rejected samples,
\begin{equation}
\label{eq:cond1}
\alpha = \alpha (r).
\end{equation}

\subsubsection{Property II: Performance measure is able to compare different rejector mechanisms working at the same rejected fraction}
\label{prop2}
For the same classification $C$, and for two different rejection mechanisms $R_1$ and $R_2$, the performance measures $\alpha(C,R_1 ,r)$ for $R_1$ and $\alpha(C,R_2,r)$ for $R_2$ should be able to compare the rejection mechanisms $R_1$ and $R_2$ when rejecting the same fraction:
\begin{align}
\label{eq:cond2}
\overbrace{\alpha (C, R_1, r)}^{\textrm{rejection } R_1} > \overbrace{\alpha (C, R_2, r)}^{\textrm{rejection } R_2} \iff  R_1 \textrm{ outperforms } R_2 .
\end{align}

\subsubsection{Property III: Performance measure is able to compare different rejector mechanisms working at different rejected fractions}
\label{prop3}
On the other hand, it is also desired that the performance measure be able to compare the performance of different rejection mechanisms $R_1$ and $R_2$ when they reject different fractions $r_1$ and $r_2$,
\begin{equation}
\label{eq:cond3}
R_1 \textrm{ outperforms } R_2 \implies \overbrace{\alpha (C, R_1, r_1)}^{\textrm{rejection } R_1} > \overbrace{\alpha (C, R_2, r_2)}^{\textrm{rejection } R_2}. 
\end{equation}

\subsubsection{Property IV: Maximum and minimum values for performance measures}
\label{prop4}
Any performance measure should achieve its maximum when $\bm \N$ coincides with $\mathcal{A}$ and $\bm \R$ with $\mathcal{M}$, corresponding to simultaneously rejecting all misclassified samples and not rejecting any accurately classified samples ($\avector_{\bm \N} = 0$ and $\avector_{\bm\R} = 1$ are empty).
Similarly, the performance measure should achieve its minimum when $\bm \N$ coincides with $\mathcal{M}$ and $\bm \R$ with $\mathcal{A}$, corresponding to rejecting all accurately classified samples and not rejecting any misclassified samples ($\avector_{\bm \N} = 1$ and $\avector_{\bm\R} = 0$ are empty).

\section{Performance measures}
\label{sec:measures}
We are now ready to define the three performance measures.
First, we will show that the nonrejected accuracy, as used extensively in the literature, is a performance measure that satisfies all our properties.
We will then present two other measures that also satisfy the same properties: classification quality and rejection quality.

\subsection{Nonrejected accuracy $\A$}
The nonrejected accuracy measures the accuracy on the subset of nonrejected samples
\begin{equation*}
\A = \frac{\|\avector_{\bm \N } \| }{ n - k} = \frac{\|\avector_{\bm \N } \| }{| \bm \N |} = \frac{\nrejc}{\nrej}.
\end{equation*}
The nonrejected accuracy measures the proportion of samples that are accurately classified and not rejected compared to the samples that are not rejected.
In a probabilistic interpretation, it is equivalent to the expected value of the conditional probability of a sample being accurately classified given that it was not rejected.

We can represent the nonrejected accuracy as a function of the rejected fraction,
\begin{equation}
\label{eq:aasr}
\A =  \frac{\|\avector_{\bm \N } \| }{| \bm \N |} = \frac{\|\avector_{\bm \N } \| }{n ( 1 -r )} = \A (r),
\end{equation}
satisfying Property I.
Properties II, III, and IV are also satisfied; the proof is given in the Appendix.

We note that the maximum and minimum values of the nonrejected accuracy, $1$ and $0$ respectively, are nonunique.
Two different rejectors can have a nonrejected accuracy of $1$ if the nonrejected samples are all accurately classified.
For example, if rejector $R_1$ rejects all misclassified samples and does not reject any accurately classified samples, $\mathcal{R} = \mathcal{M}$, and rejector $R_2$ rejects all misclassified samples and some accurately classified samples, $\mathcal{R} \supseteq \mathcal{M}$, both their nonrejected accuracies will be $1$.

\subsection{Classification quality $\Q$}
The classification quality measures the correct decision making of the classifier-rejector, assessing both the performance of the classifier on the set of nonrejected samples and the performance of the rejector on the set of misclassified samples.
This equates to measuring the number of accurately classified samples not rejected $\mathcal{A} \cap \mathcal{N}$ and the number of misclassified samples rejected $\mathcal{M} \cap \mathcal{R}$,
\begin{equation*}
\Q = \frac{\|\avector_{\bm \N } \|  + \|\mathbf{1} - \avector_{\bm \R } \| }{|\bm \N|+|\bm \R|} = \frac{\|\avector_{\bm \N } \|  + \|\mathbf{1} - \avector_{\bm \R } \| }{n} = \frac{\qtop}{\all}.
\end{equation*}
In a probabilistic interpretation, this is equivalent to the expected value of probability of a sample being accurately classified and not rejected or a sample being misclassified and rejected.

To represent the classification quality $Q$ as a function of the fraction of rejected samples $r$, we analyze separately the performance of the classifier on the subset of nonrejected samples and the performance of the rejector on the subset of misclassified samples.
The performance of the classifier on the subset of nonrejected samples is the proportion of accurately classified samples not rejected to the total number of samples, which can be easily represented in terms of the nonrejected accuracy as follows,
\begin{equation}
\label{eq:q_part1}
\frac{\|\avector_{\bm \N } \|}{n}  = \frac{\|\avector_{\bm \N } \|}{n (1-r)}  (1-r)=A(r)  (1-r).
\end{equation}
The performance of the rejector on the subset of misclassified samples is
\begin{align}
\label{eq:q_part2}
 \frac{ \|\mathbf{1} - \avector_{\bm \R } \| }{n}  = \frac{ \|\mathbf{1} - \avector\| - \|\mathbf{1} - \avector_{\bm \N}\|  }{n} = 1-A(0) - \frac{ \|\mathbf{1}- \avector_{\bm \N} \|}{ n}  = \nonumber \\
 1-A(0) -  \frac{ k - \| \avector_{\bm \N} \|}{n}  = 1 -A(0)  - \frac{k}{n} +\frac{\| \avector_{\bm \N} \|}{n}  = \nonumber \\
 1 - A(0) - (1 - r) + A(r) (1-r)  = - A(0) + r  + A(r) (1-r).
\end{align}
By combining \eqref{eq:q_part1} and \eqref{eq:q_part2}, we can represent the classification quality as 
\begin{equation}
\label{eq:qasr}
Q(r) = 2 A(r) (1-r) + r - A(0),
\end{equation}
satisfying Property I.
Properties II, III, and IV are also satisfied; the proof is given in the Appendix.

We note that both the maximum and the minimum values of the classification quality, $1$ and $0$ respectively, are unique.
$Q(r) = 1$ describes an ideal rejector that does not reject any of the accurately classified samples and rejects all misclassified samples, $\mathcal{A} = \mathcal{N}$ and $\mathcal{M} = \mathcal{R}$.
Conversely, $Q(r) = 0$ describes the worst rejector that rejects all the accurately classified samples and does not reject any misclassified sample, $\mathcal{A} = \mathcal{R}$ and $\mathcal{M} = \mathcal{N}$.

We can use the classification as in~\eqref{eq:qasr} to compare the proportion of correct decisions between two different rejectors, for different values of rejected fractions.
We note that as $Q(0) = A(0)$, we can compare the proportion of correct decisions by using classification with rejection versus the use of no rejection at all.

\subsection{Rejection quality $\PH$}
Finally, we present the rejection quality to evaluate the ability of the rejector to reject misclassified samples.
This is measured through the ability to concentrate all misclassified samples onto the rejected portion of samples.
The rejection quality is computed by comparing the proportion of misclassified to accurately classified samples on the set of rejected samples with the proportion of misclassified to accurately classified samples on the entire data set,
\begin{equation*}
\phi = \frac{\|\mathbf{1}- \avector_{\bm \R} \|}{\|\avector_{\bm \R} \|} \bigg/ \frac{\|\mathbf{1} - \avector \|}{\|\avector \|} = \frac{\rnc}{\rc} \bigg/ \frac{\nc}{\cor}
\end{equation*}
As the rejection quality is not defined when there are no misclassified rejected samples, $|\avector_{\bm \R}\| = 0$, we define $\phi \equiv \infty $ if any sample is rejected $|\R| >0$, meaning that no accurately classified sample is rejected and some misclassified samples are rejected, and $\phi \equiv 1$ if no sample is rejected $|\R| =0$.
To express the rejection quality as a function of the rejected fraction, we note that, by \eqref{eq:fundamental_identity}, we can represent the accuracy on the rejected fraction as $\|\avector_{\bm \R} = \|\avector\| - \|\avector_{\bm \N } \|$, and $\| 1-\avector \|$ as $n(1 - A(0))$.
This means that 
\begin{equation*}
\phi = \frac{r - A(0) + A(r) (1-r) }{A(0) - A(r) (1-r)}\frac{A(0)}{1 - A(0)} = \phi(r)
\end{equation*}
satisfying Property I.
Properties II, III, and IV are also satisfied; the proof is given in the Appendix.

Unlike the nonrejected accuracy and the classification quality, the rejection quality is unbounded.
A value of $\phi$ greater than one means that the rejector is effectively decreasing the concentration of misclassified samples on the subset of nonrejected samples, thus increasing the nonrejected accuracy.

The minimum value of $\phi$ is $0$, and its maximum is unbounded by construction.
Any rejector that only rejected misclassified samples will achieve a $\phi$ value of $\infty$, regardless of not rejecting some misclassified samples.

\section{Quantifying performance of a classifier with rejection}
\label{sec:relative_optimality}
With the three performance measures defined, we can now compare performance of classifiers with rejection.
We illustrate this in Fig.~\ref{fig:outperformance_nonrej}, where we consider a general classifier with rejection.
In the figure, black circles in the center correspond to a classifier that rejects $20\%$ of the samples, with a nonrejected accuracy of $62.5\%$, a classification quality of $65\%$, and a rejection quality of $3.67$; we call that black circle a \emph{reference operating point}.

\subsection{Reference operating point, operating point and operating set}
A set of performance measures and the associated rejected fraction $r$ correspond to a \emph{reference operating point} of the classifier with rejection. 
Given a reference operating point, we define the \emph{operating set} as the set of achievable operating points as a function of the rejected fraction.
This further means that for each operating point of a classifier with rejection there is an associated operating set.

Any point in the green region of each of the plots in Fig.~\ref{fig:outperformance_nonrej} is an operating point of a classifier with rejection that outperforms the one at the reference operating point (black circle), and any operating point in the orange region is an operating point of a classifier with rejection that is outperformed by the one at the reference operating point (black circle), regardless of the cost function (assuming that the cost of rejection is never greater than the cost of misclassification).
In white regions, performance \emph{depends} on the trade-off between rejection and misclassification, and is thus dependent of the cost function.
The borders of the green and orange regions correspond to the best and worst behaviors, respectively, of classifiers with rejection as compared to the reference operating point.
Thus, given the reference operating point, its correspondent \emph{operating set} is the union of the white regions including the borders.

\begin{figure*}[ht!]
\begin{center}
\begin{tabular}{ccc}
\hspace{-0.2in}\includegraphics[width=.35\columnwidth]{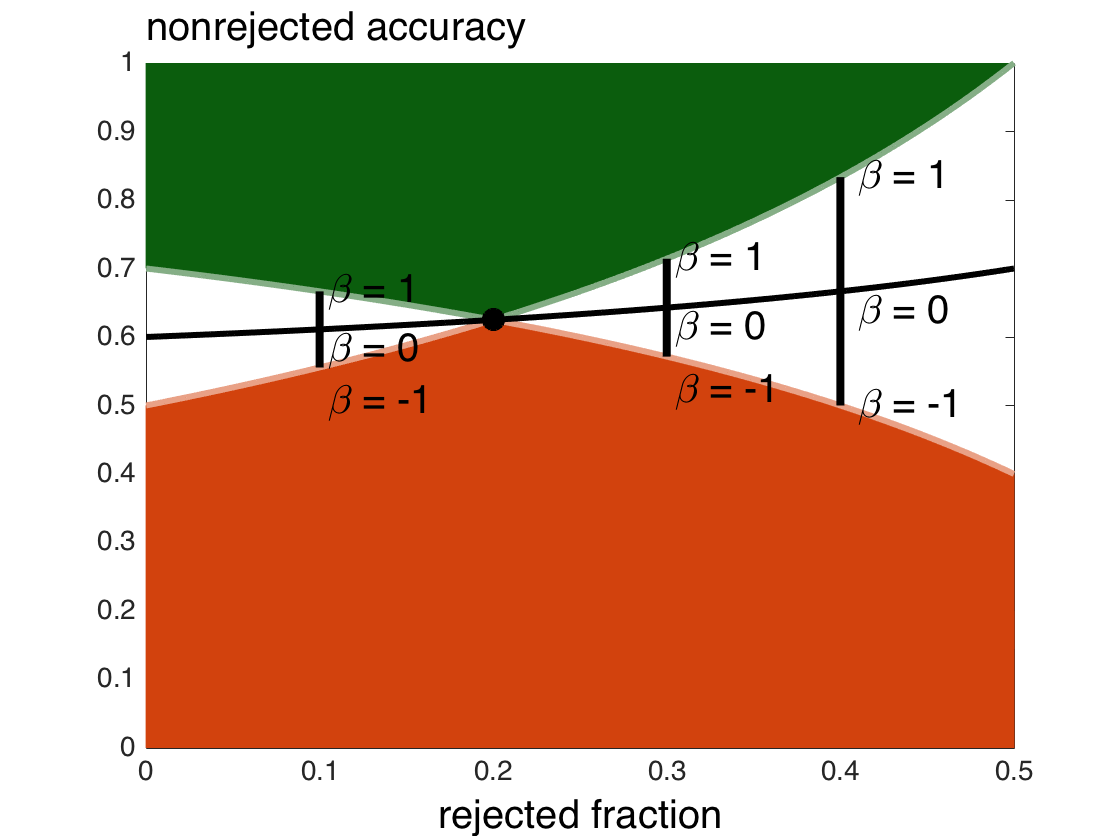}&
\hspace{-0.2in}\includegraphics[width=.35\columnwidth]{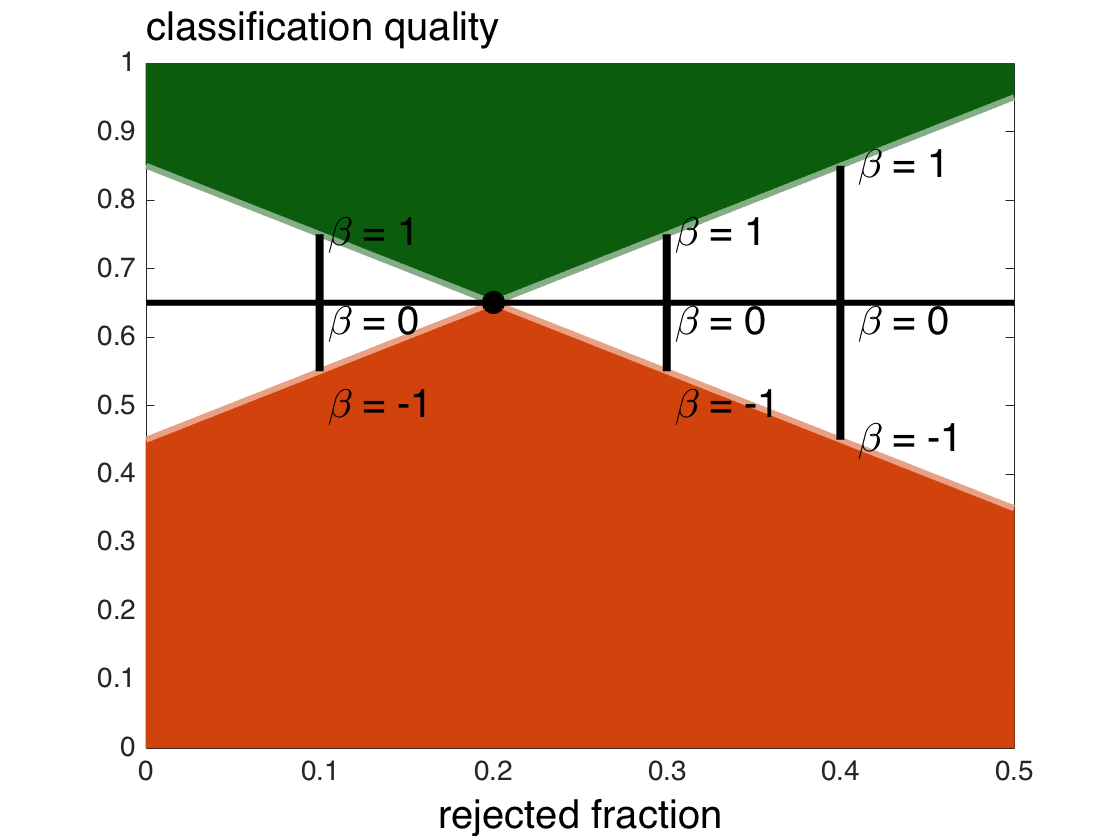} &
\hspace{-0.2in}\includegraphics[width=.35\columnwidth]{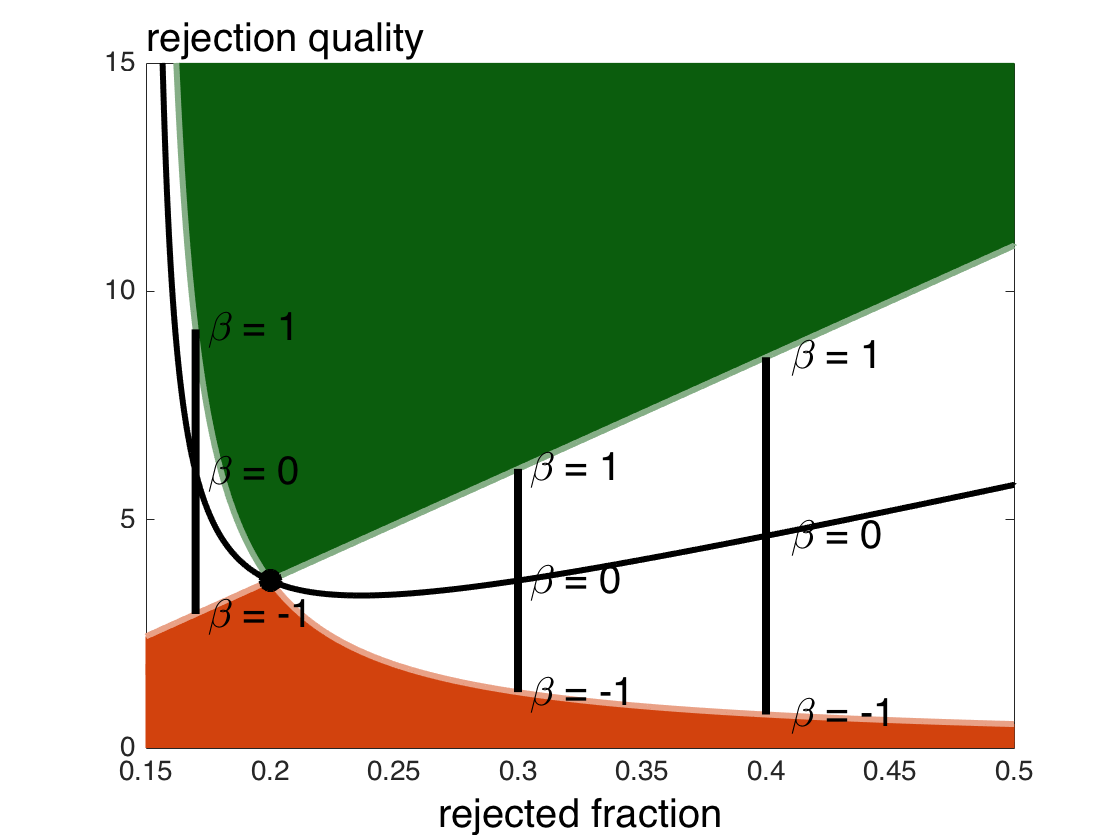}\\
(a) Nonrejected accuracy&
(b) Classification quality&
(c) Rejection quality\\
\end{tabular}
\caption{\label{fig:outperformance_nonrej}
Performance measures with outperformance (green) underperformance (orange) regions for a reference operating point (black circle).
Reference classifier rejects $20\%$ of the samples and achieves a nonrejected accuracy of $62.5\%$, classification quality of $65\%$, and a rejection quality of $3.67$. $\beta $ measures correctness of rejection;
$\beta = 1$ corresponds to the best and
$\beta = -1$ to the worst rejection behaviors, respectively.
}
\end{center}
\end{figure*} 

\subsection{Relative optimality}
To compare the behavior of a classifier with rejection in the white region to that at the reference operating point, we measure how close that classifier is to the green and orange region borders (best/worst behaviors).
Let $\beta = 0$ denote the curve that corresponds to the middle point between the best and worst behaviors (black curve in Fig.~\ref{fig:outperformance_nonrej}), $\beta = 1$ to the best behavior (border with the green region), and $\beta = -1 $ to the worst behavior (border with the orange region).
We call $\beta$ \emph{relative optimality}, as it compares the behavior of a classifier with rejection relative to a given reference operating point.

Let us consider a reference operating point defined by a nonrejected accuracy $A_0$ at a rejected fraction $r_0$; we can now compare the performance at an arbitrary operating point $(A_1,r_1)$ with that at a reference operating point $(A_0,r_0)$ as
\begin{equation}
\label{eq:beta_r1r0}
\beta = 
\begin{cases}
 2 \frac{A_1 (1-r_1) - A_0 (1-r_0)}{r_1 - r_0} + 1,& \textrm{ if } r_1 > r_0 , \\
- 2 \frac{A_1 (1-r_1) - A_0 (1-r_0)}{r_1 - r_0} - 1 ,& \textrm{ if } r_1 < r_0.
\end{cases}
\end{equation}

\subsection{Cost function}
Furthermore, the relative optimality allows us to compare any two operating points of a classifier with rejection taking in account a cost function $L$ which measures the relative cost of rejection versus misclassification.
Let us consider the following generic cost function
\begin{equation}
\label{eq:loss_function}
L_\rho (\hat{y}_i^R) = \begin{cases}
0,& \textrm{ $\hat{y}_i^R$ accurately classified and not rejected;}\\
1,& \textrm{ $\hat{y}_i^R$ misclassified and not rejected;}\\
\rho,& \textrm{ $\hat{y}_i^R$ rejected,}
\end{cases}
\end{equation}
where $\rho$ is the cost of rejection and represents the trade-off between rejection and misclassification.
We now connect the concept of relative optimality with the generic cost function $L$ as follows.
\begin{theo}
For an operating point $(A_1,r_1)$ with a relative optimality of $\beta$ relative to the reference operating point $(A_0,r_0)$, and $r_1 > r_0$, 
\begin{equation}
\label{eq:loss_relative_th}
\sgn(\Delta L_\rho) = \sgn(L_\rho(A_0,r_0) - L_\rho(A_1,r_1)) = \sgn\left(\frac{\beta + 1}{2} - \rho\right),
\end{equation}
where $\Delta L_\rho$ is the difference between the cost function at the reference operating point $L_\rho(A_0,r_0)$ and the cost function at the operating point $L_\rho(A_1,r_1)$.
\end{theo}
\begin{proof}
Let $r_1 > r_0$; then we have that the cost function at a generic operating point $(A,r)$ is
\begin{equation*}
L_\rho(A,r) = (1-r)(1-A) n + \rho r n,
\end{equation*}
as we have $(1-r)(1-A) n$ misclassified samples, $(1-r) A n$ accurately classified samples, and $r n$ rejected samples,
and thus 
\begin{align}
\Delta L_\rho & = n \left( (1-r_0)(1-A_0) + \rho r_0 - (1-r_1)(1-A_1) - \rho r_1 \right) \nonumber \\
 \label{eq:rho_pre}& = n \left(r_1 - r_0 - (1-r_0) A_0 + (1-r_1) A_1 + \rho (r_0 - r_1) \right).
\end{align}
On the other hand, from \eqref{eq:beta_r1r0}, we have that
\begin{equation}
\label{eq:beta_gtr0}
A_1 ( 1-r_1) - A_0 (1-r_0) = \frac{\beta - 1}{2} (r_1 - r_0).
\end{equation}
By combining \eqref{eq:rho_pre} and \eqref{eq:beta_gtr0}, we have that
\begin{equation}
\Delta L_\rho = n (r_1 - r_0) \left(  \frac{\beta + 1}{2} - \rho \right).
\end{equation}
Because $r_1-r_0$ and $n$ are positive, $\Delta L_\rho$ and $(\beta + 1) / 2 - \rho $ have the same sign.
\end{proof}
The previous discussion allows us to compare a classifier with rejection $R_1$ to the reference operating point $R_0$ as follows.
Let the operating point $(A_1,r_1)$ be at relative optimality $\beta$ with respect to the reference operating point $(A_0,r_0)$.
Then,
\begin{equation}
\label{eq:connect_loss}
\begin{cases}
L_\rho(A_1,r_1) < L_\rho(A_0,r_0),& \textrm{ for } \rho < (\beta + 1)/2; \\
L_\rho(A_1,r_1) \geq L_\rho(A_0,r_0),& \textrm{ for } \rho \geq (\beta + 1)/2. \\
\end{cases}
\end{equation}

\subsection{Performance measures}
\label{subsec:reconstruction}
We can consider the classifier with rejection as two coupled classifiers if we considered the rejector $R$ to be a binary classifier on the output $\hat{y}$ of the classifier $C$, assigning to each sample a rejected or nonrejected label.
Ideally, $R$ should classify as rejected all samples misclassified by $C$ and classify as nonrejected all the samples accurately classified by $C$.

In this binary classification formulation, the classification quality $Q$ becomes the \emph{accuracy} of the binary classifier $R$, the accuracy of the nonrejected samples $A$ becomes the \emph{precision} (positive predictive value) of the binary classifier $R$, and the rejection quality $\phi$ becomes the \emph{positive likelihood ratio} (the ratio between the true positive rate and the false positive rate) of the binary classifier $R$.
The rejected fraction becomes the ratio between the number of samples classified as rejected and the total number of samples.

This formulation allows us to show that the triplet $(A(r), Q(r), r)$ completely specifies the behavior of the rejector by relating the triplet to the confusion matrix associated with the binary classifier $R$.
As we are able to reconstruct the confusion matrix from the triplet, we are thus able to show that the triplet $(A(r), Q(r), r)$ is sufficient to describe the behavior of the rejector.

\begin{theo}
The set of measures $(A(r), Q(r), r)$ completely specifies the behavior of the rejector.
\end{theo}
\begin{proof}
Let us consider the following confusion matrix associated with $R$:
\begin{align}
\left[ \begin{array}{cc} |\mathcal{A}\cap\mathcal{N}| & |\mathcal{M}\cap\mathcal{N}|   \\ |\mathcal{A}\cap\mathcal{R}|   & |\mathcal{M}\cap\mathcal{R}| \end{array} \right], \nonumber
\end{align}
where $n$ denotes the total number of samples, $|\mathcal{A}\cap\mathcal{N}|/n$ the number of samples accurately classified and not rejected, $\frac{ |\mathcal{M}\cap\mathcal{N}|}{n}$ the number of samples misclassified but not rejected, $|\mathcal{A}\cap\mathcal{R}| $ the number of samples accurately classified but rejected, and $|\mathcal{M}\cap\mathcal{R}|$ the number of samples misclassified and rejected.
Given that $n$ binary classifications classified $n$ samples, the confusion matrix associated with $R$ can be uniquely obtained from the following full rank system:
\begin{align}
\left[\begin{array}{c}  |\mathcal{A}\cap\mathcal{N}| \\ |\mathcal{M}\cap\mathcal{N}|   \\ |\mathcal{A}\cap\mathcal{R}|   \\ |\mathcal{M}\cap\mathcal{R}| \\\end{array}\right] =
n \left[\begin{array}{cccc}0 & 0& 0 &1 \\
1 & -1 & 0 & -1 \\
 0 & 0 & 1 & 1 \\
0 & 1 & -1 & 1 \\
\end{array} \right]
\left[\begin{array}{c} 1 \\  r \\ Q(r) \\ A(r) ~ (1-r) \end{array}\right] .
\nonumber
\end{align}
Therefore, as the set of measures and the confusion matrix are related by a full-rank system, the set of measures $(A(r),Q(r), r)$ completely specifies describes the behavior of the rejector.
\end{proof}

\subsection{Comparing performance of classifiers with rejection}
Given a classifier $C$ and two rejectors $R_1$ and $R_0$, with $r_1> r_0$, and a cost function with a rejection-misclassification trade-off $\rho$, we can now compare the performance of classifiers with rejection.

Rejector $R_1$ outperforms $R_0$ when the following equivalent conditions  are satisfied:
\begin{align}
A_{R_1}(r_1)  > A_{R_0}(r_0) \frac{1-r_0}{1-r_1} + (\rho - 1) \frac{r_1-r_0}{1-r_1} \iff 
Q_{R_1}(r_1)  >  Q_{R_0}(r_0)  + (2 \rho - 1 )(r_1 -r_0). \nonumber
\end{align}

Rejector $R_0$ outperforms $R_1$ when the following equivalent conditions are satisfied:
\begin{align} A_{R_1}(r_1)  < A_{R_0}(r_0) \frac{1-r_0}{1-r_1} + (\rho - 1) \frac{r_1-r_0}{1-r_1} \iff
Q_{R_1}(r_1)  <  Q_{R_0}(r_0)  + (2 \rho - 1 )(r_1 -r_0). \nonumber
\end{align}

Rejectors $R_0$ and $R_1$ are equivalent in terms of performance when the following equivalent conditions are  satisfied:
\begin{align}
A_{R_1}(r_1)  = A_{R_0}(r_0) \frac{1-r_0}{1-r_1} + (\rho - 1) \frac{r_1-r_0}{1-r_1} \iff
Q_{R_1}(r_1)  =  Q_{R_0}(r_0)  + (2 \rho - 1 )(r_1 -r_0) . \nonumber
\end{align}
The proof is given in the Appendix.

\section{Experimental results}
\label{sec:experimental}
To illustrate the use of the proposed performance measures, we apply them to the analysis of the performance of classifiers with rejection applied to synthetic and real data.
We use a simple synthetic problem to motivate the problem of classification with rejection and to serve as a toy example.

We then focus on the application of classification with rejection to pixelwise hyperspectral image classification~\cite{CondessaBK:15a,CondessaBK:15c,CondessaBK:15d}, which is prone to the effects of small and  nonrepresentative training sets, meaning that the classifiers might not be equipped to deal with all existing classes, due to the potential presence of unknown classes.
Classification with rejection is an interesting avenue for hyperspectral image classification as the need to accurately classify the samples is greater than the need to classify all samples.

\subsection{Synthetic data}
As a toy example, we consider a classification problem consisting of four two-dimensional Gaussians with the identity matrix as a covariance matrix and centers at $(\pm 1, \pm 1)$.
The Gaussians overlap significantly, as shown in Fig.\ref{fig:synthetic_data}(a).
This results in a simple classification decision: for each sample, assign the label of the class with the closest center as in Fig.\ref{fig:synthetic_data}(b).

We illustrate our performance measures by comparing two simple rejection mechanisms: (1) \emph{maximum probability rejector}, which, given a classifier and a rejected fraction, rejects the fraction of samples with lower probability; and (2) \emph{breaking ties rejector}, which, given a classifier and a rejected fraction, rejects the fraction of samples with lower difference between the highest and second-highest class probabilities.

\begin{figure}[h!]
\begin{center}
\begin{tabular}{cc}
\hspace{-0.2in}\includegraphics[width=.5\columnwidth]{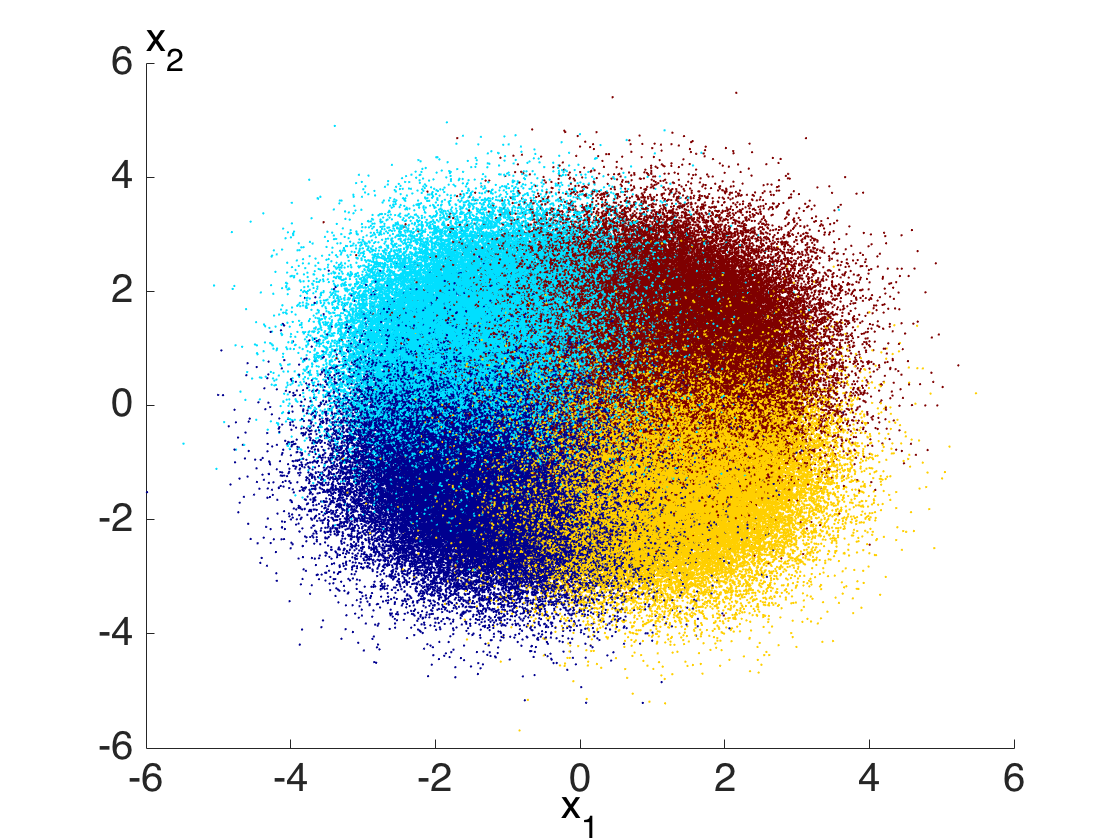} &
\hspace{-0.2in}\includegraphics[width=.5\columnwidth]{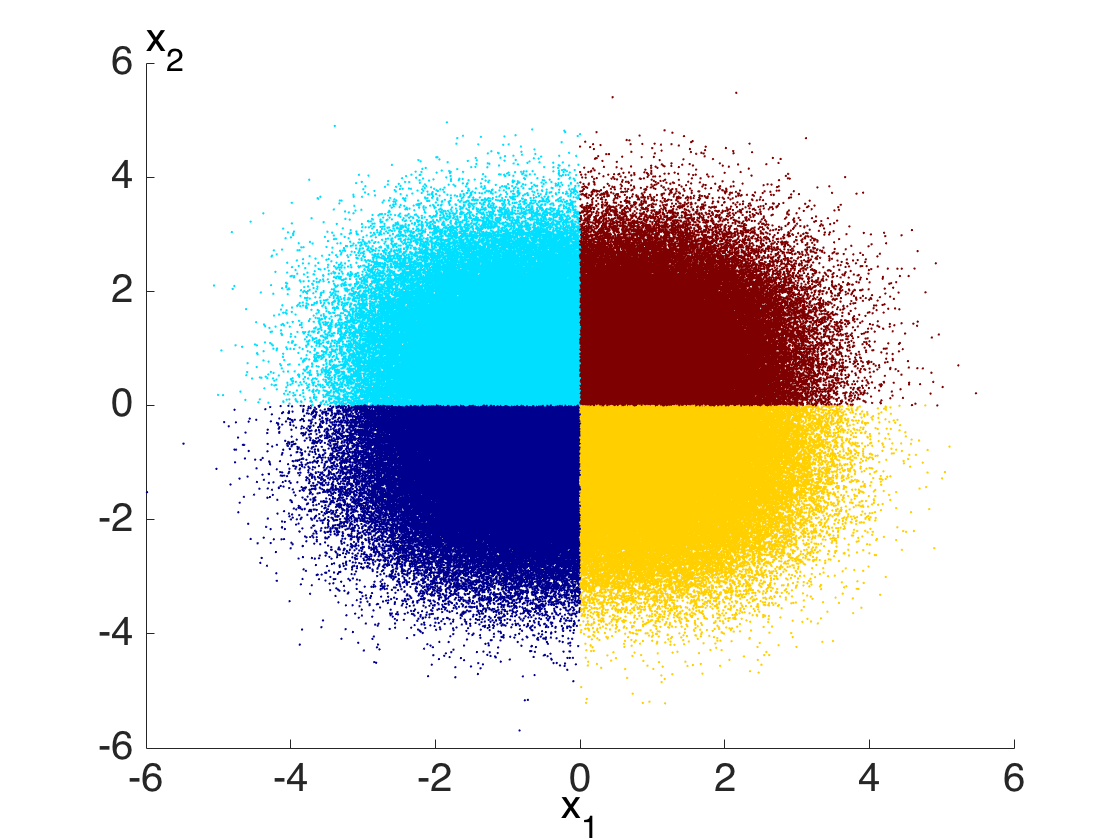} \\
\fns (a) Ground truth & \fns(b) Classification no rejection \\
\hspace{-0.2in}\includegraphics[width=.5\columnwidth]{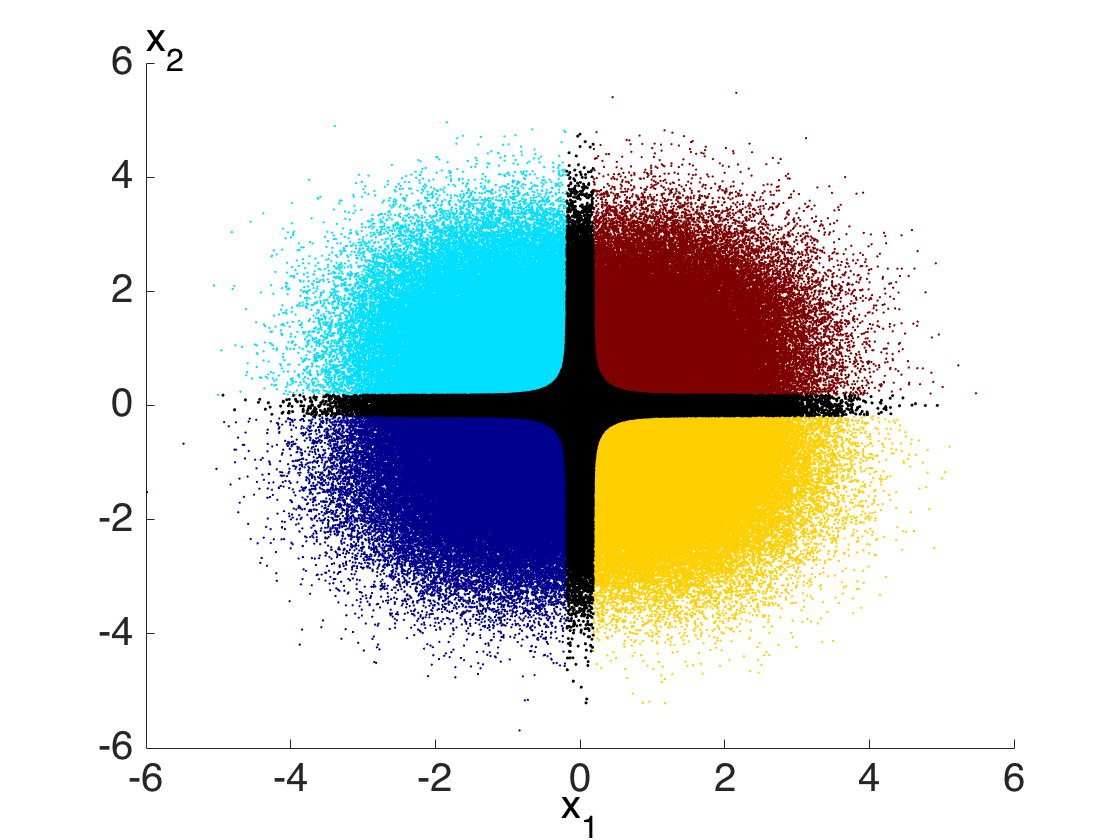} &
\hspace{-0.2in}\includegraphics[width=.5\columnwidth]{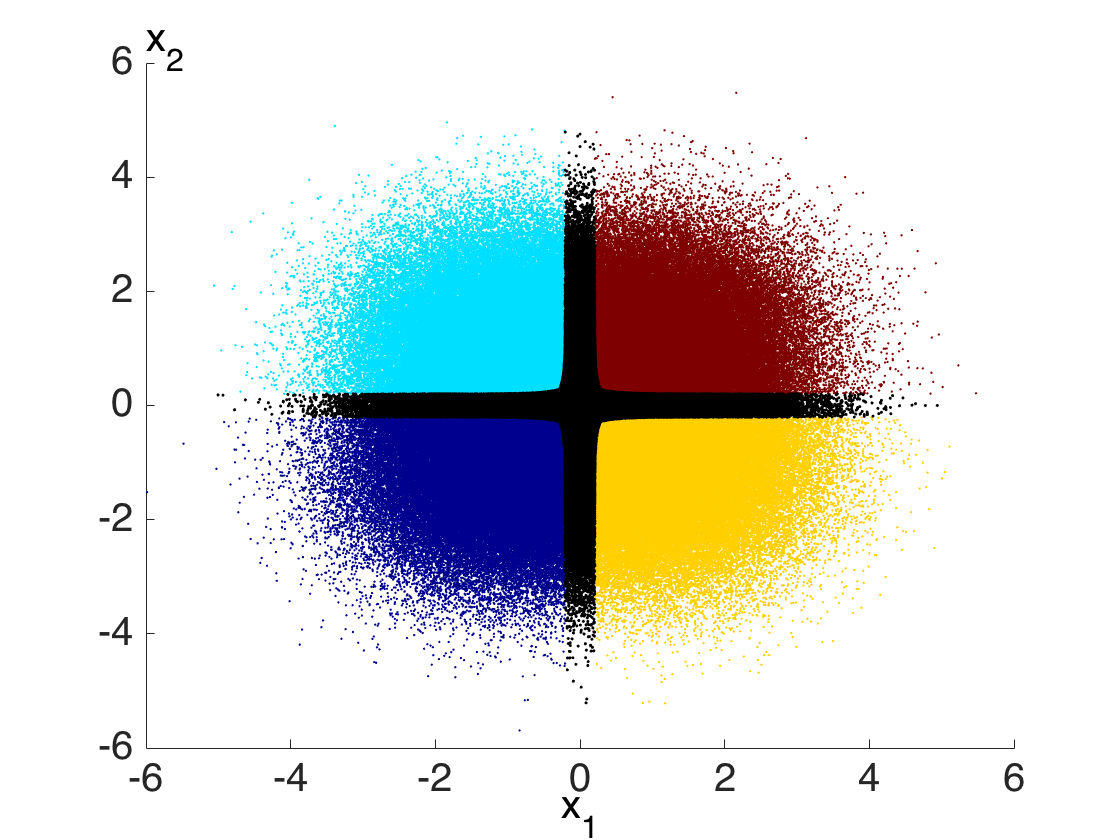} \\
\fns(c) Classification $20\%$ rejection & \fns (d) Classification $20\%$ rejection \\
\fns maximum probability & \fns breaking ties\\
\end{tabular}
\caption{\label{fig:synthetic_data}
Synthetic data example. Four Gaussians with equal covariance (identity covariance matrix) and significant overlap (centered at $(\pm 1, \pm 1)$), classified with rejection (in black).
(a) Ground truth, (b) classification with no rejection, (c) classification with $20\%$ rejection using maximum probability rejector, and (d) classification with $20\%$ rejection using breaking ties rejector.
The differences between the two rejectors are clear near the origin.
Note that the points are not uniformly distributed.
}
\end{center}
\end{figure} 

In Fig.\ref{fig:synthetic_data_measures}, we can see the performance measures computed for all possible rejected fractions for each of the two rejectors.
It is clear that the with the accuracy-rejection curves alone, as shown in Fig.\ref{fig:synthetic_data_measures}(a), we are not able to single out any operating point of the classifier with rejection.
On the other hand, with the classification quality in Fig.\ref{fig:synthetic_data_measures}(b), we can identify where the rejector is maximizing the number of correct decisions, and for which cases having a reject option outperforms not having rejection.
As illustrated in Fig.\ref{fig:synthetic_data_measures}(c), the rejection quality provides an easy way to discriminate between two different rejectors, as it focuses on the analysis of the ratios of correctly classified to incorrectly classified samples on the set of rejected samples.
\begin{figure*}[htpb]
\begin{center}
\begin{tabular}{ccc}
\hspace{-0.2in} \includegraphics[width=.35\columnwidth]{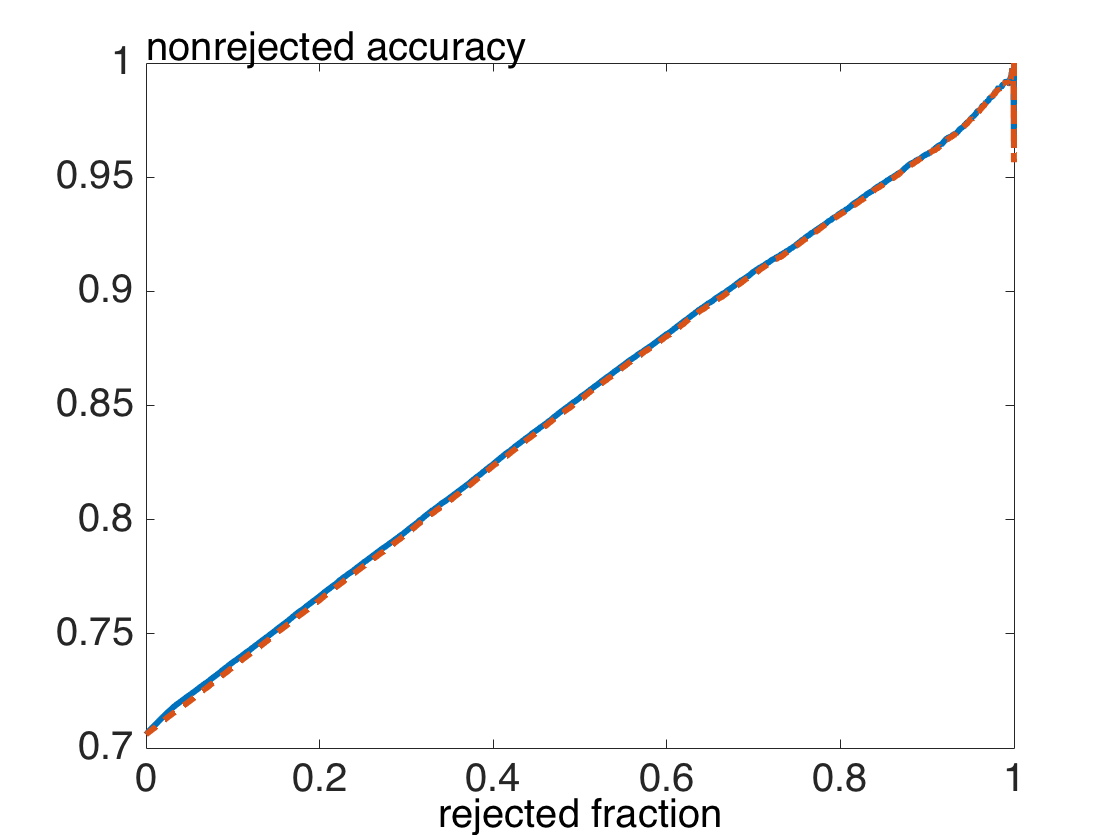} &
\hspace{-0.2in} \includegraphics[width=.35\columnwidth]{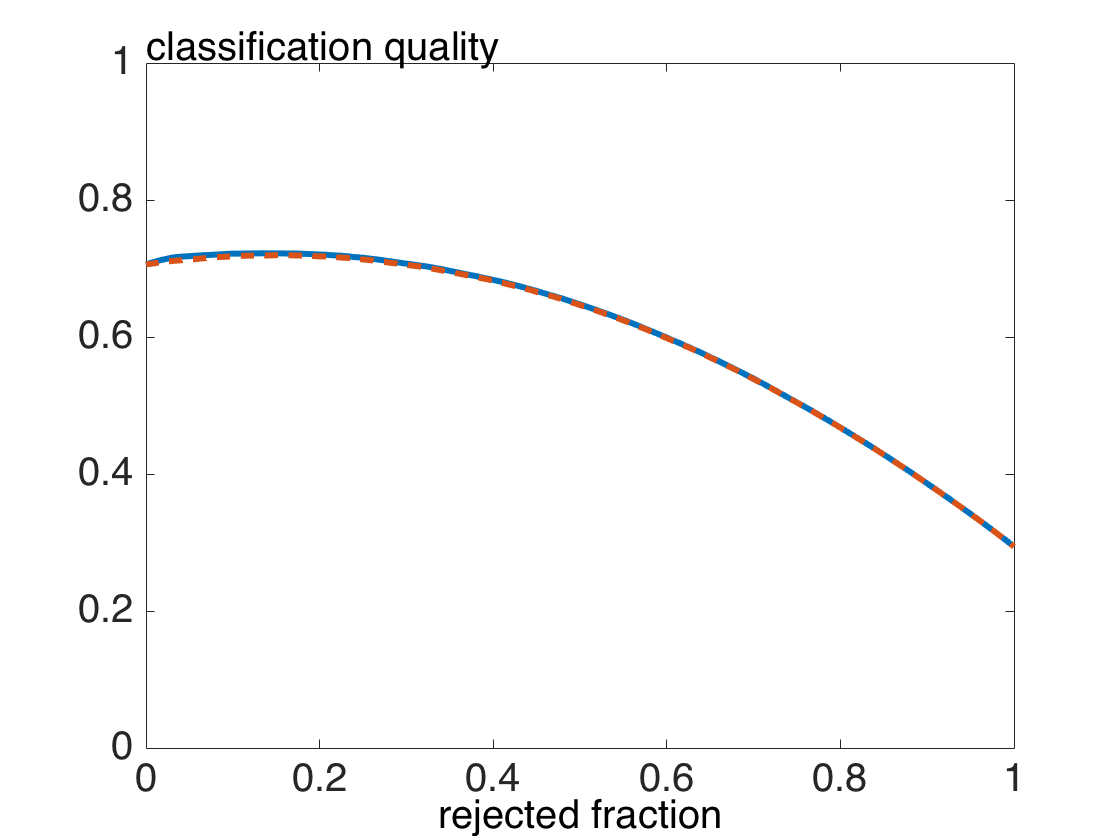}&
\hspace{-0.2in} \includegraphics[width=.35\columnwidth]{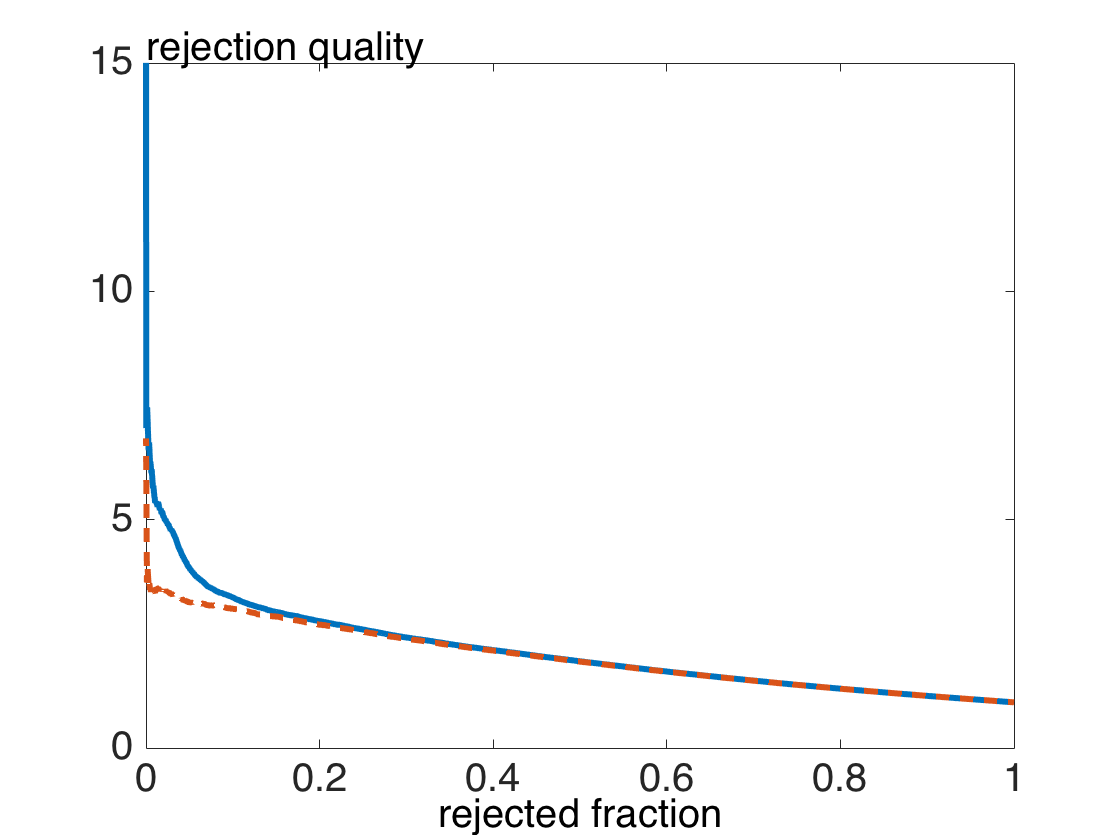} \\
\fns (a) Nonrejected accuracy &
\fns (b) Classification quality &
\fns (c) Rejection quality  \\
\end{tabular}
\caption{\label{fig:synthetic_data_measures}
Performance measures as a function of the rejected fraction for the synthetic example and the maximum probability rejector (solid blue line), and the breaking ties (dashed red line).
}
\end{center}
\end{figure*}

\begin{figure}[htpb]
\begin{center}
\begin{tabular}{cc}
\includegraphics[width=.5\columnwidth]{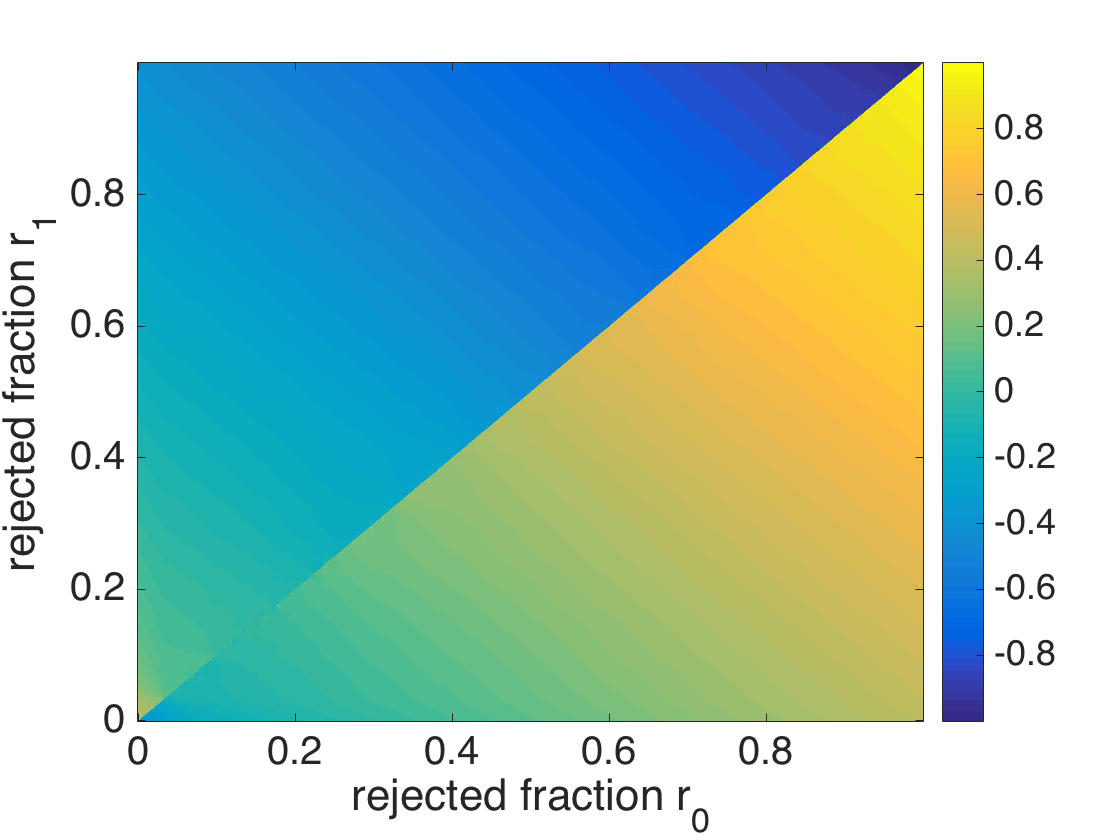} &
\hspace{-0.5cm}\includegraphics[width=.5\columnwidth]{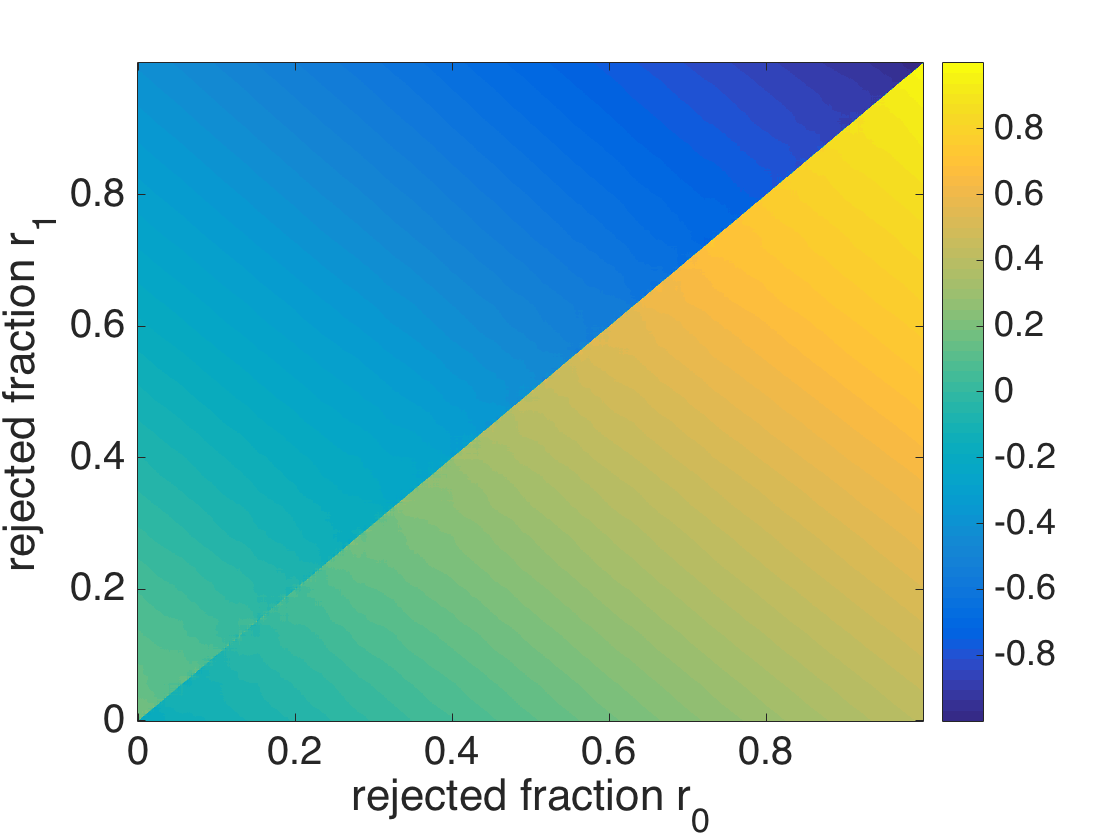} \\ 
\fns (a) \emph{Maximum probability rejector} &
\fns (b) \emph{Breaking ties rejector} 
\end{tabular}
\caption{\label{fig:synthetic_data_rel_opt}
Relative optimality computed for all possible pairs of operating points of (a) \emph{maximum probability rejector} and (b) \emph{breaking ties rejector}
}
\end{center}
\end{figure}

The relative optimality plots for both rejectors are present in Fig.~\ref{fig:synthetic_data_rel_opt}. For each possible operating point of the rejector, for simplicity defined only by the rejected fraction, we compute the relative optimality of all other operating points of the rejector.
We note that, for both rejectors, the operating point that corresponds to the maximum classification quality, has a nonnegative relative optimality with regards to all other operating points.
This relative optimality plot is of particular interest for parameter selection.

\subsection{Hyperspectral image data}
In hyperspectral image classification, the use of context, through the form of spatial priors, is widespread, providing significant performance improvements.
This means that, after classification, a computationally  expensive procedure is applied to classifier output to take into account contextual effects.
The use of accuracy-rejection curves might not be feasible, as changes in the rejected fraction often imply a computationally expensive context computation procedure. 
Thus, due to the joint use of context and rejection, and the high computational costs associated, this is a perfect environment for the use of the performance measures.

\begin{figure}[htb]
\begin{center}
\begin{tabular}{cc}
\includegraphics[width=.45\columnwidth]{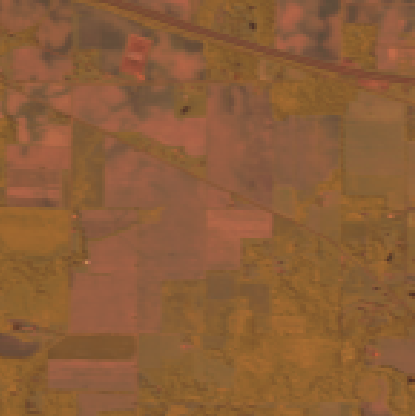} &
\includegraphics[width=.45\columnwidth]{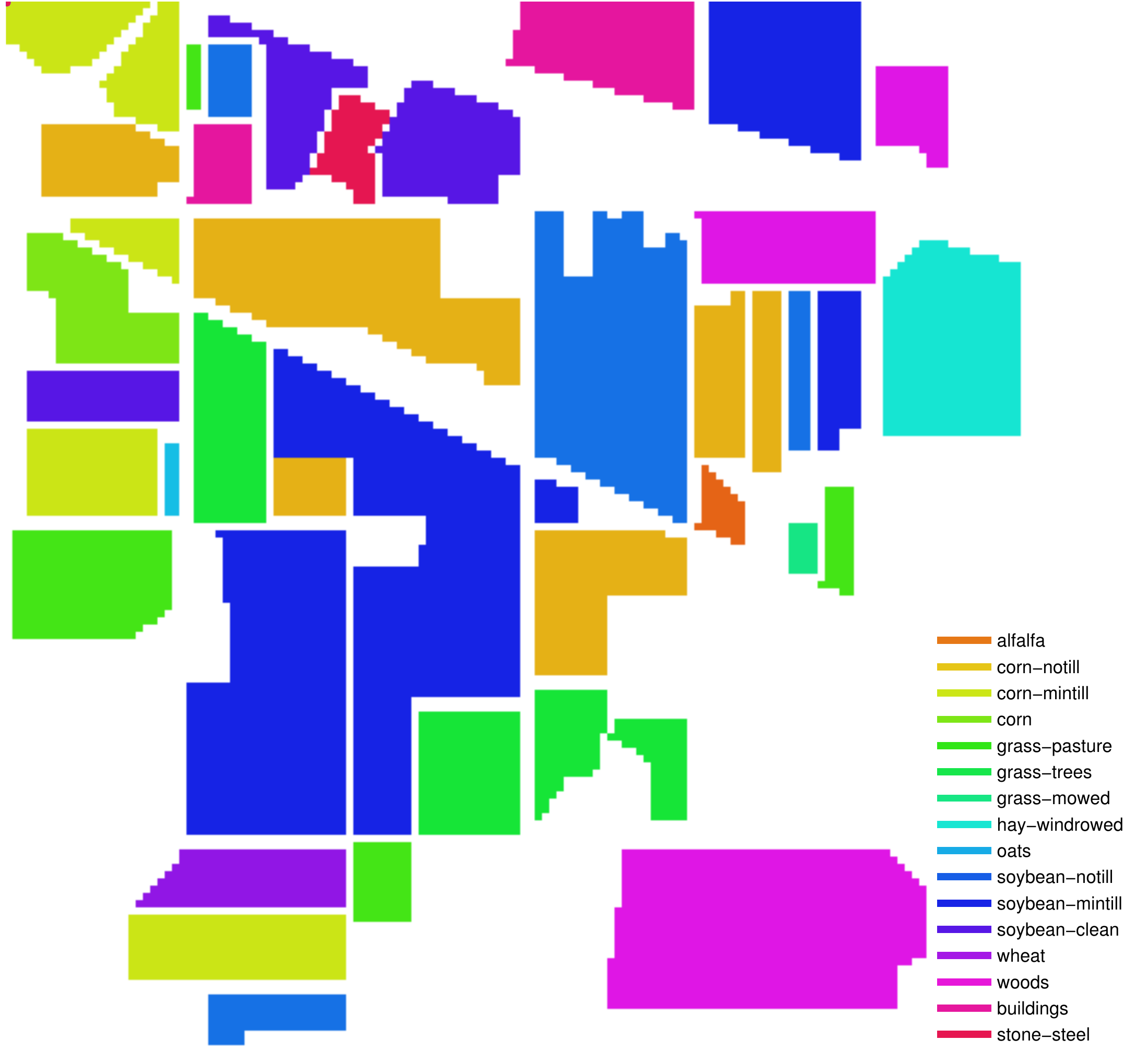} \\
\fns (a) False color composition  &
\fns (b) Ground truth \\
\includegraphics[width=.45\columnwidth]{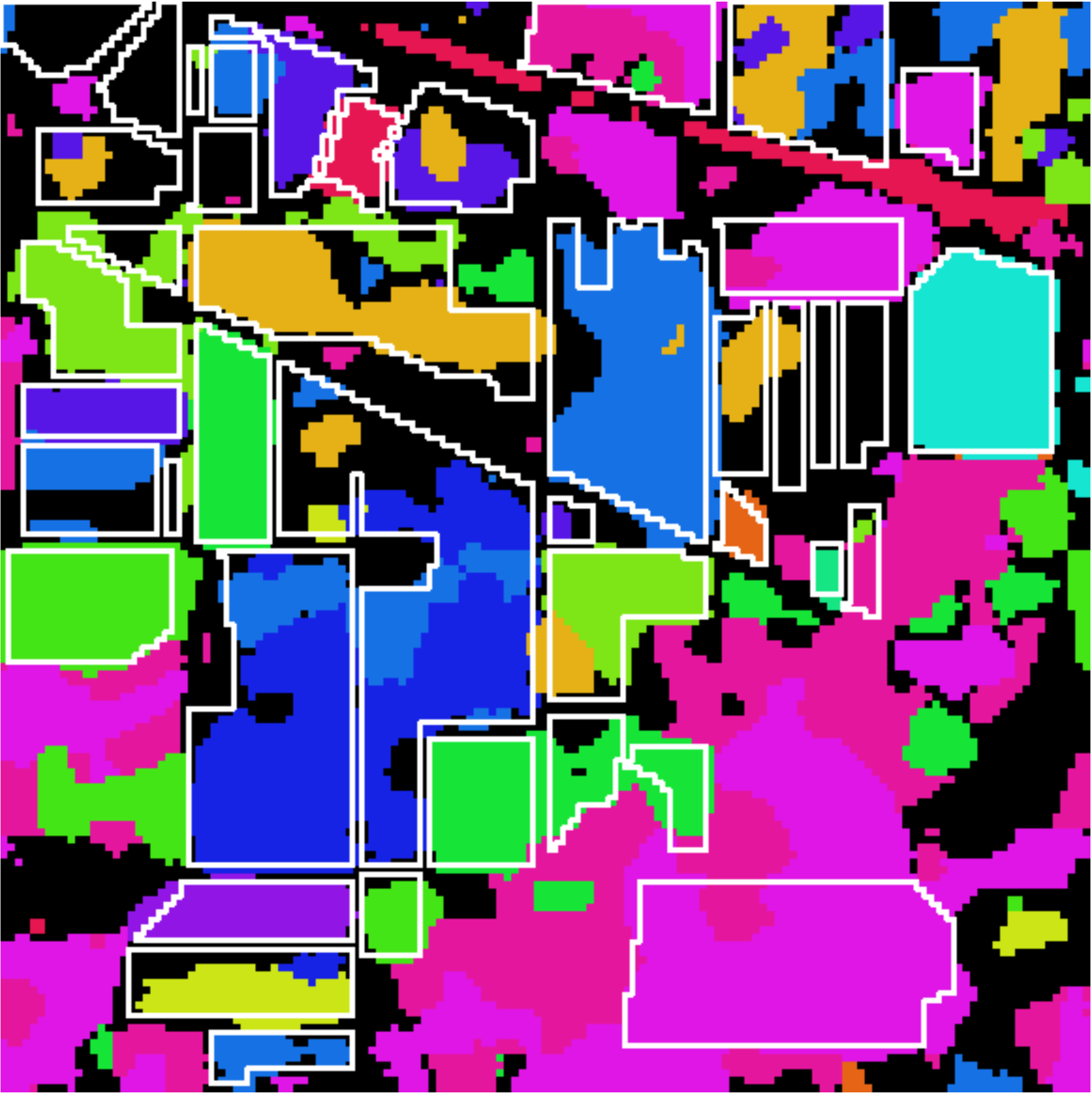} &
\includegraphics[width=.45\columnwidth]{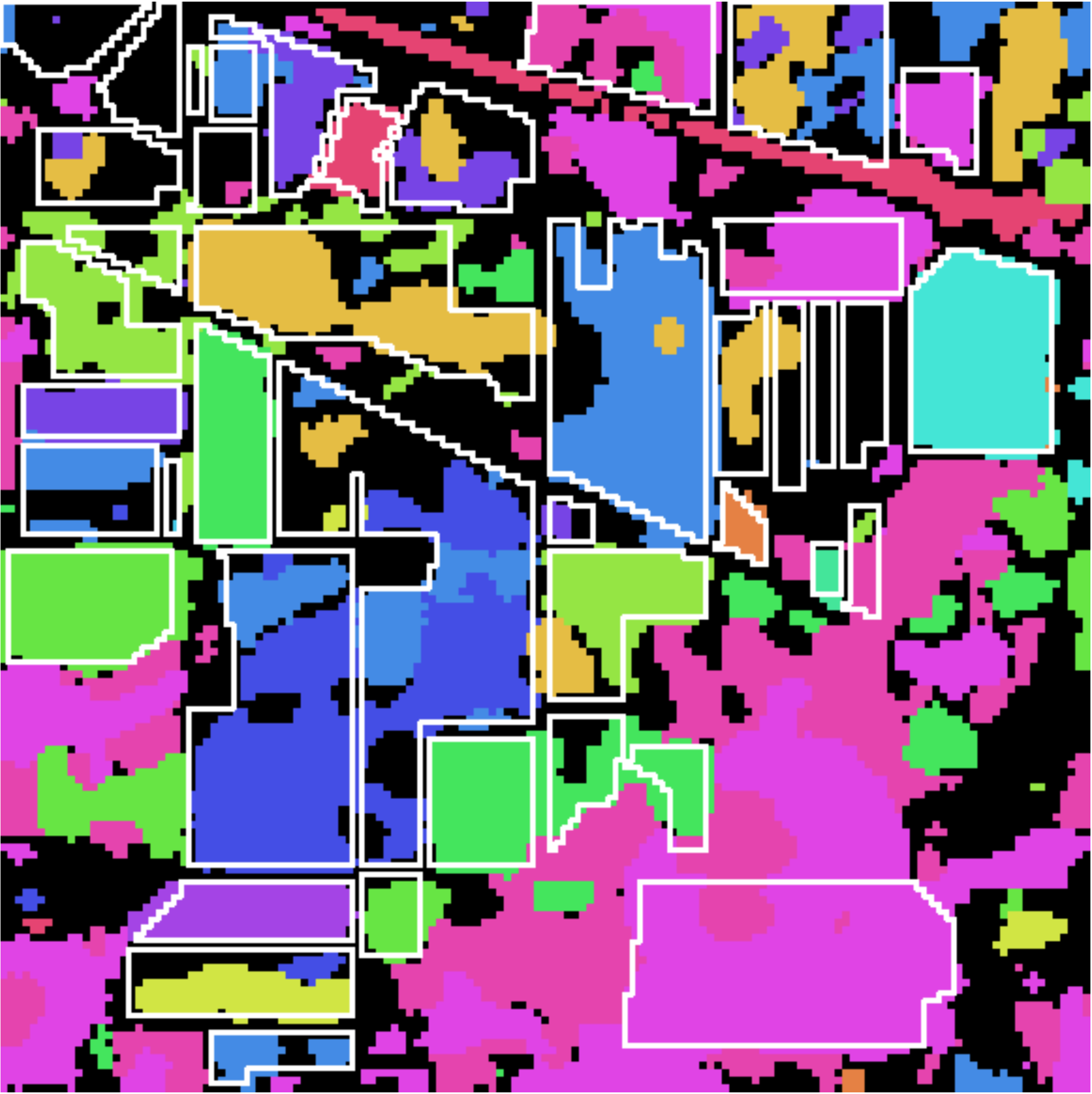} \\
\fns (c) SegSALSA-JCR &
\fns (d) SegSALSA-SCR \\
\end{tabular}
\caption{\label{fig:real_data_figs}
Indian Pine scene. (a) False color composition, (b) ground truth, (c) classification with context and rejection with SegSALSA-JCR, and (d) classification with context and rejection with SegSALSA-SCR.
}
\end{center}
\end{figure}

We use the algorithms for hyperspectral image classification with context and rejection presented in ~\cite{CondessaBK:15d} in their joint (JCR) and in their sequential (SCR), versions, respectively.
Both JCR and SCR are based on SegSALSA (Segmentation by Split Augmented Lagrangian Shrinkage Algorithm).
SegSALSA consists of a soft supervised classifier assigning a class probability to each pixel of the image, followed by the application of context through the computation of the marginal maximum \emph{a posteriori} of a continuous hidden field that driving the class probabilities, with a smoothness promoting prior applied on the continuous hidden field.
For more details on the SegSALSA, see~\cite{BioucasDiasCK:14}.

\begin{figure*}[htb]
\begin{center}
\begin{tabular}{ccc}
\hspace{-0.2in} \includegraphics[width=.35\columnwidth]{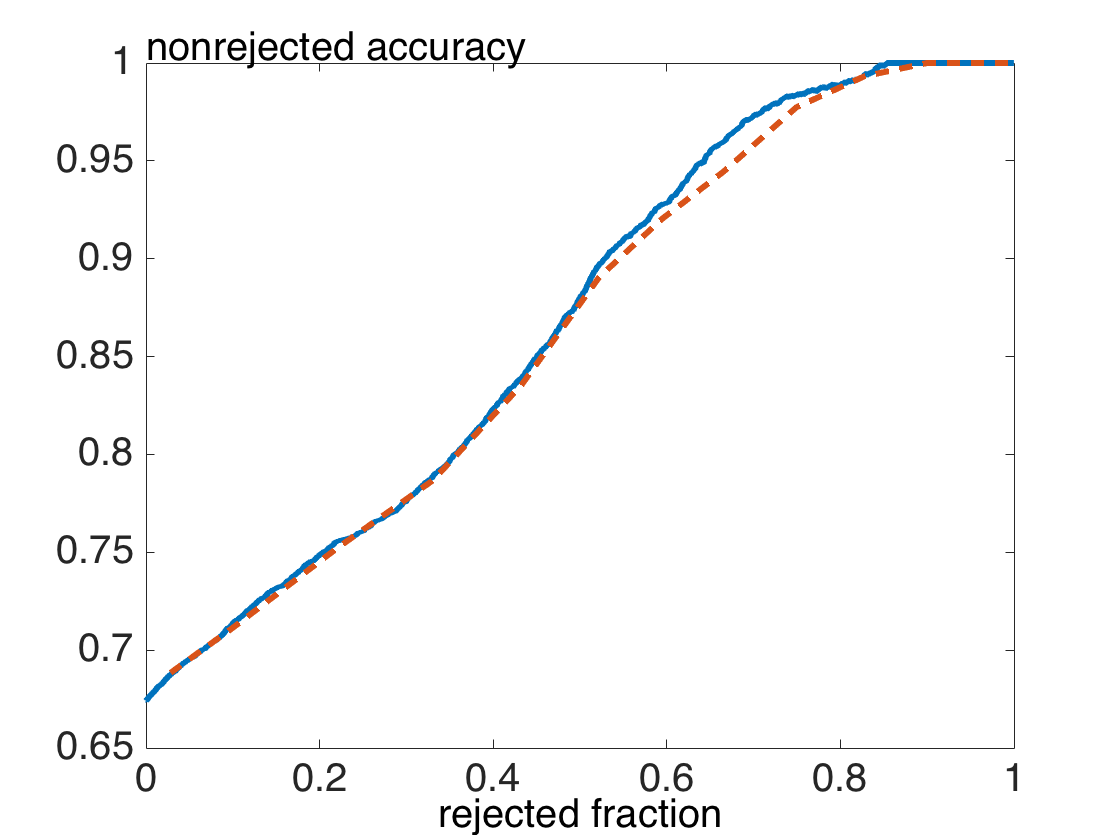} &
\hspace{-0.2in} \includegraphics[width=.35\columnwidth]{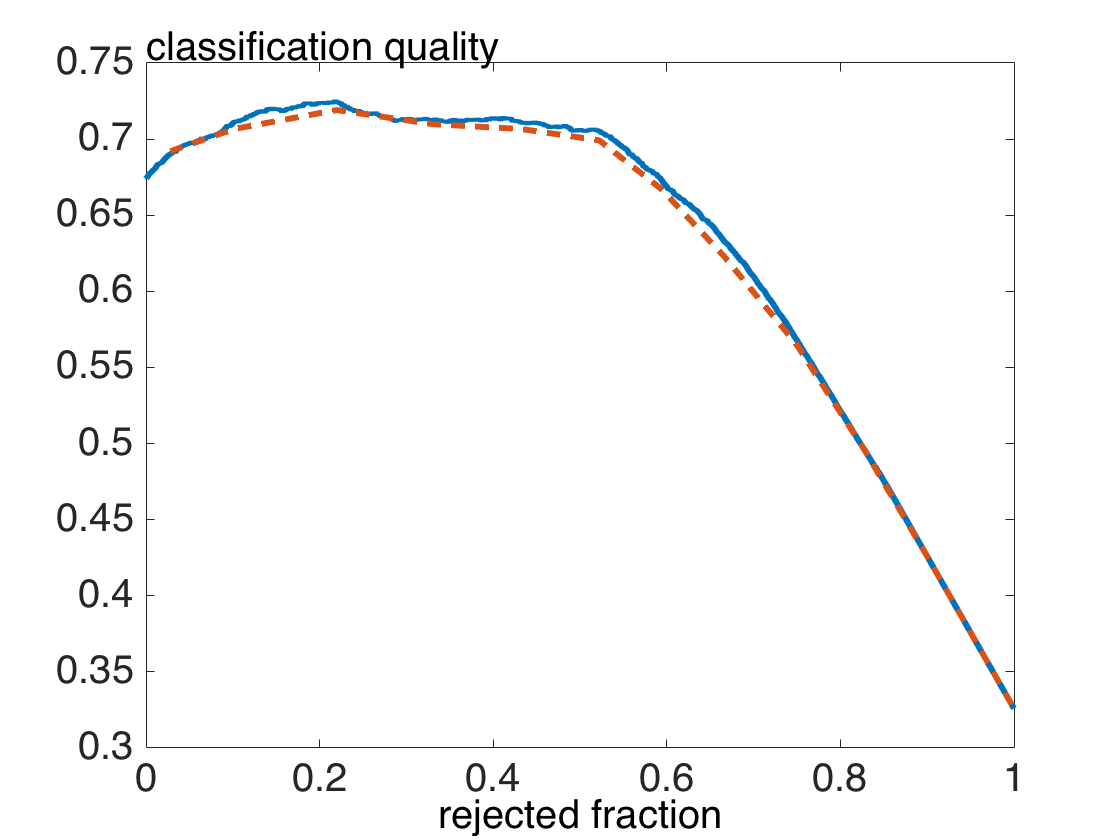}&
\hspace{-0.2in} \includegraphics[width=.35\columnwidth]{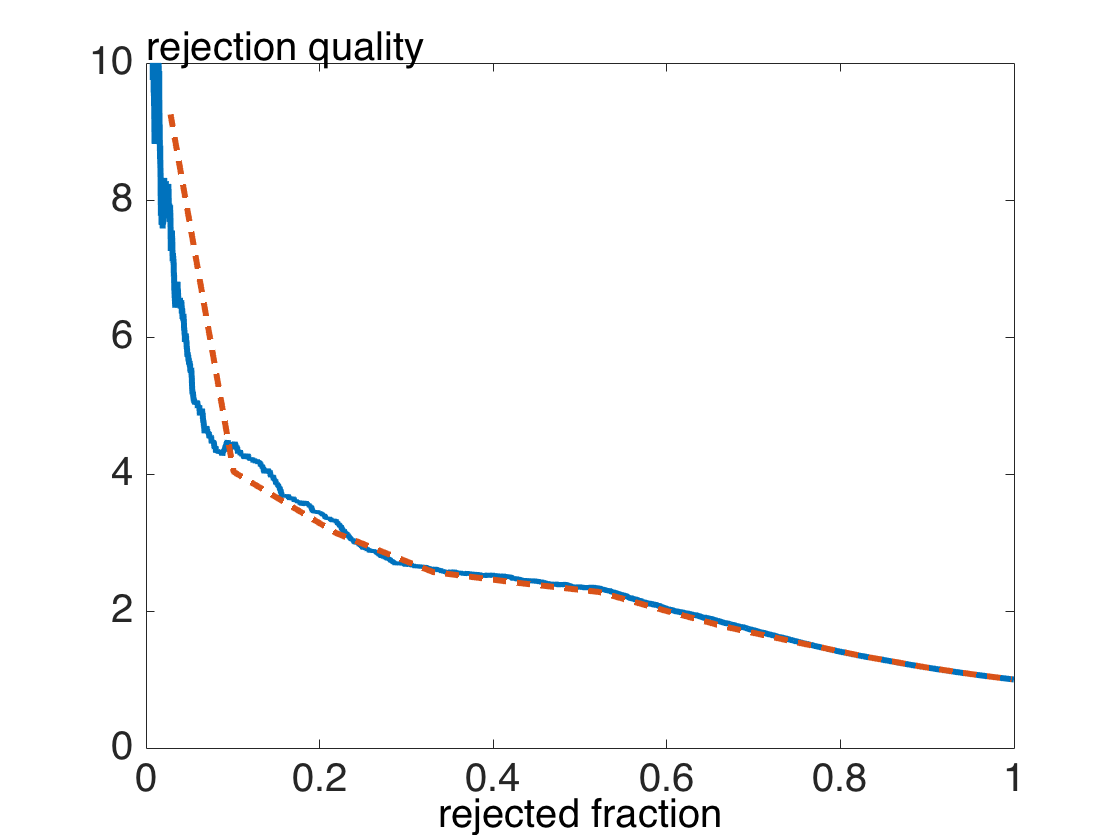} \\
\fns (a) Nonrejected accuracy  &
\fns (b) Classification quality &
\fns (c) Rejection quality \\
\end{tabular}
\caption{\label{fig:real_data_measures}
Performance measures as a function of the rejected fraction for the Indian Pine scene and the SegSALSA-JCR rejector (dashed red line), and SegSALSA-SCR rejector (solid blue line).
}
\end{center}
\end{figure*} 

We can now introduce rejection as an extra class that models probability of classifier failure, resulting in a joint computation of context and rejection, the JCR version.
We consider that the probability of failure is constant for all the pixels of the hyperspectral image, leading to a uniform weighting of the samples.
The higher the probability of failure, the larger the rejected fraction.
However, it is not possible to define \emph{a priori} the amount of rejected fraction obtained, and any change in the value of rejected fraction implies the recomputation of the SegSALSA algorithm.

On the other hand, we can harness the hidden fields resulting from the SegSALSA algorithm to obtain an ordering of the pixels in the classification with context according to their confidence.
This results in a very fast rejection scheme that takes in account rejection, resulting from approximations to the problem of joint computation of context and rejection, following a sequential approach to context and rejection, the SCR version.

We apply the SegSALSA-JCR and the SegSALSA-SCR to the classification of a well known benchmark image in the hyperspectral community, the AVIRIS Indian Pine scene\footnote{We thank Prof. Landgrebe at Purdue University for providing the AVIRIS Indian Pines scene to the community.}, as shown in Fig.~\ref{fig:real_data_figs}.
The scene consists of a $145\times 145$ pixel section with $200$ spectral bands (the water absorption bands are removed) and contains $16$ nonmutually exclusive classes.

Following the approach in~\cite{CondessaBK:15d}, we learn the class models using a sparse logistic regression with a training set composed of $10$ samples per class, and, for the joint approach, perform a parameter sweep on the probability of classifier failure, obtaining various operating points of the JCR rejector.
For the SCR rejector, as we define \emph{a posteriori} the rejected fraction, obtaining operating points of the SCR rejector is simply obtained by rejecting the fraction of pixels with the least amount of confidence (smaller value of the posterior probability on of the hidden field).
See~\cite{CondessaBK:15d} for a detailed explanation of the JCR and SCR schemes for rejection with context.

As seen in Fig.~\ref{fig:real_data_measures}, it is clear that, by looking at the accuracy rejection curves alone, it is trivial to compare the performance of the two rejectors when working at the same rejected fraction.
However, we cannot draw any conclusions on which is the best operating point of each rejector, or how they compare to each other.
By looking at the classification quality, it is clear where the maximum number of correct decisions is made for each of the rejectors, and by looking at the rejection quality we can observe that there is a significant improvement with reject options for lower values of the rejected fraction.

\begin{figure}[htb]
\begin{center}
\begin{tabular}{cc}
\includegraphics[width=0.5\columnwidth]{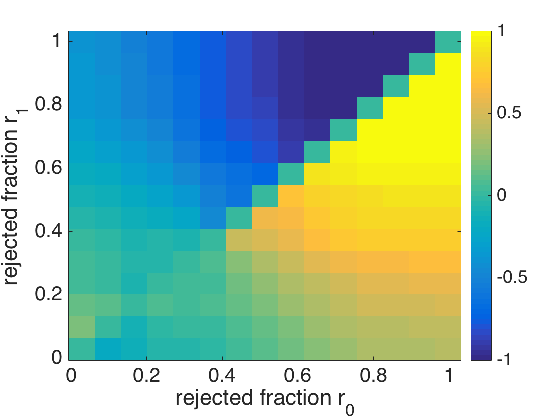} &
\includegraphics[width=0.5\columnwidth]{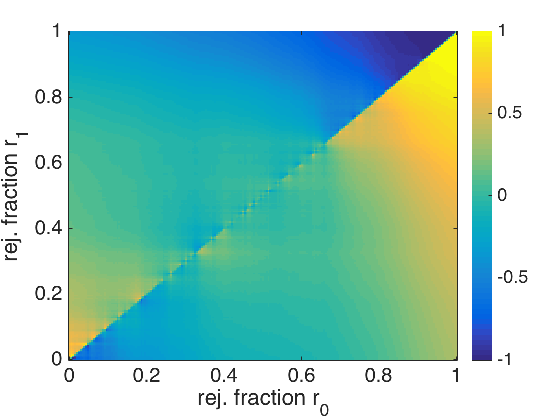} \\
\fns (a) Relative optimality JCR &
\fns (b) Relative optimality SCR \\
\includegraphics[width=0.5\columnwidth]{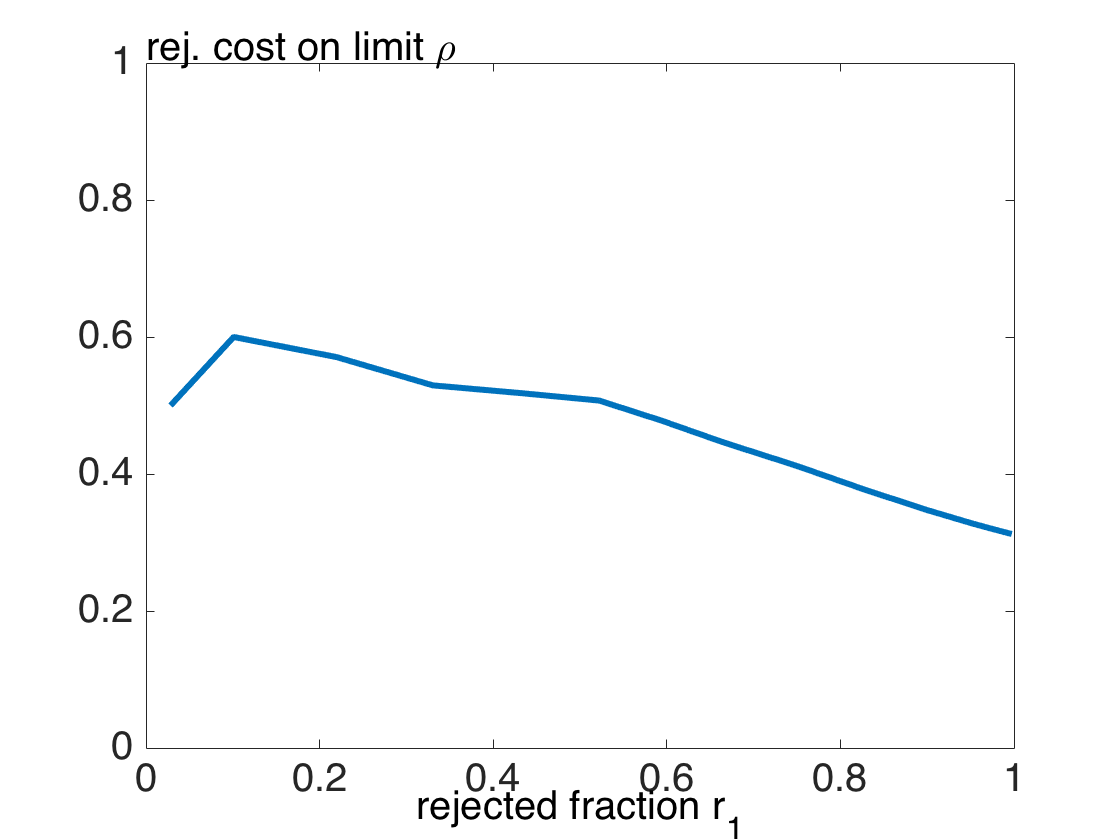} &
\includegraphics[width=0.5\columnwidth]{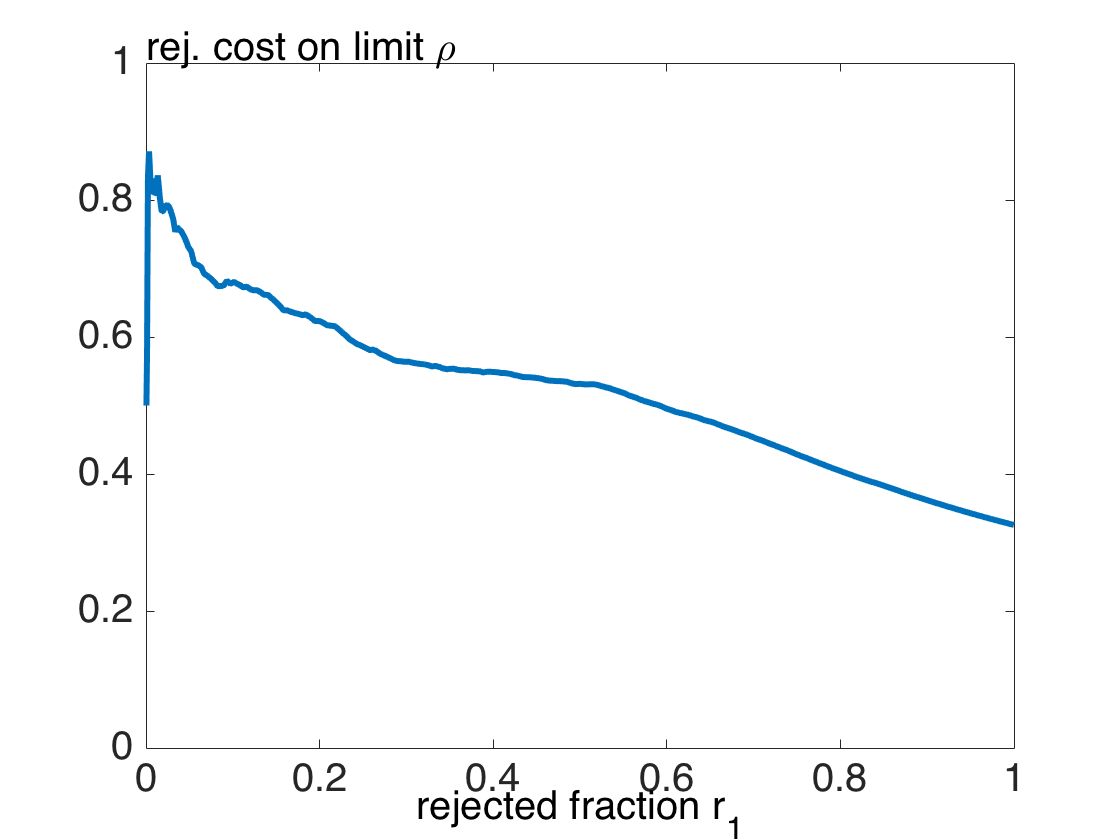} \\
\fns (c) Minimum rejection cost for &
\fns (d) Minimum rejection cost for \\
\fns no rejection JCR &
\fns no rejection SCR \\
\end{tabular}
\caption{\label{fig:real_data_rel_opt}
Relative optimality computed for all possible pairs of operating points of (a) JCR and (b) SCR rejectors, and minimum value of $\rho$ in the cost function~\eqref{eq:loss_function}, for (c) JCR and (d) SCR rejectors, such that it is better \emph{not} to reject ($r_0 = 0$), for each operating point.
}
\end{center}
\end{figure} 

Fig.~\ref{fig:real_data_rel_opt} shows  the relative optimality between each pair of operating points for each of the rejectors.
Using \eqref{eq:loss_relative_th}, for a given reference operating point and for any test operating point, we can obtain the minimum value of $\rho_0$ in the cost function~\eqref{eq:loss_function} such that the cost function at the test operating point is smaller than than the cost function at the reference operating point.
This means that, for any cost function with $\rho < \rho_0$, the test operating point is better than the reference operating point.
We perform such analysis in Fig.~\ref{fig:real_data_rel_opt}, where we set the reference operating point as $r_0 =0$, meaning no rejection.
For each possible value of rejected fraction $r_1$, we then test what the minimum value of the $\rho$ such that \emph{no rejection} is a better option than rejecting a fraction $r_1$ would be, according to the operating points defined by the two rejectors.


\section{Conclusions}
\label{sec:conclusion}
We introduced a set of measures to quantify performance of classifiers with rejection.
We then applied these performance measures to classifiers with rejection on both synthetic and real-world (hyperspectral image) data.
Furthermore, we connected the performance measures presented with general cost functions through the concept of relative optimality.


\bibliographystyle{IEEEbib}

\clearpage
\section{Appendix}

\subsection{Properties}
Let us consider a classifier $C$ and two different rejectors $R_1$ and $R_2$.
\subsubsection{Nonrejected accuracy $\A$}
\paragraph{Property I} The nonrejected accuracy is a function of the number of rejected samples \eqref{eq:aasr}.
\paragraph{Property II} For the same rejected fraction $r$, we have that if the nonrejected accuracy for $R_1$ is greater than the nonrejected accuracy for $R_2$, then
\begin{equation*}
A_{R_1}(r) = \frac{\| \avector_{\bm \N_{R_1}} \|}{(1-r)n} > \frac{\| \avector_{\bm \N_{R_2}} \|}{(1-r)n} = A_{R_2}(r) \qed
\end{equation*} meaning $ R_1$ outperforms $R_2$.
\paragraph{Property III} If $R_1$ outperforms $R_2$, for different rejected fractions $r_1>r_2$, then $\| \avector_{\bm \N_{R_1}} \| = \| \avector_{\bm \N_{R_2}} \|$, leading to 
\begin{equation*}
A_{R_1}(r_1) = \frac{\| \avector_{\bm \N_{R_1}} \|}{(1-r_1) n} = \frac{\| \avector_{\bm \N_{R_2}} \|}{(1-r_1) n} >\frac{\| \avector_{\bm \N_{R_2}} \|}{(1-r_2) n} = A_{R_2}(r_2) \qed
\end{equation*}
If $R_1$ outperforms $R_2$, for different rejected fractions $r_1<r_2$, then $\| \bm{1} - \avector_{\bm \N_{R_1}} \| = \| \bm{1} -  \avector_{\bm \N_{R_2}} \|$.
\begin{align}
\| \bm{1} - \avector_{\bm \N_{R_1}} \| =  \| \bm{1} -  \avector_{\bm \N_{R_2}} \| & \iff 
 (1-r_1) n - \|\avector_{\bm \N_{R_1}} \| =  (1-r_2) n - \|\avector_{\bm \N_{R_2}} \| &\iff \nonumber \\
 (1 - A_{R_1}(r_1)) = \frac{ (1 - r_2)}{ (1 - r_1)}  - \frac{ (1 - r_2) \|\avector_{\bm \N_{R_2}} \|}{(1 - r_1) (1 - r_2)} & \iff 
 (1 - A_{R_1}(r_1))  =  \frac{ (1 - r_2)}{ (1 - r_1)} (1 - A_{R_2}(r_2)) & \iff \nonumber \\
 1  - A_{R_1}(r_1)  <  1 - A_{r_2}(r_2) & \iff 
 A_{R_1}(r_1) > A_{R_2}(r_2) \qed \nonumber
\end{align}

\paragraph{Property IV} The nonrejected accuracy achieves its maximum, $1$, when $\N = \mathcal{A}$ and $\mathcal{R} = \mathcal{M}$.
This maximum is not unique however.
Any selection of $\N$ such that $\N \subset \mathcal{A}$ achieves a maximum value of nonrejected accuracy. 
The minimum of the nonrejected accuracy, $0$, is achieved when $\N = \mathcal{M}$ and $\mathcal{R} = \mathcal{A}$.
Any selection of $\N$ such that $\N \subset \mathcal{M}$ achieves a minimum value of nonrejected accuracy.

\subsubsection{Classification quality $\Q$}

\paragraph{Property I}
As seen in \eqref{eq:qasr}, the classification quality is a function of the number of rejected samples.

\paragraph{Property II}
With representation of the classification quality in \eqref{eq:qasr}, we can note that, for the same rejected fraction $r$ if the classification quality for $R_1$ is higher than the classification quality for $R_2$, then
\begin{align}
Q_{R_1}(r)  > Q_{R_2}(r) &\iff 
2 A_{R_1}(r) (1-r) - A(0)  > 2 A_{R_2}(r) (1-r)  - A(0)  &\iff\nonumber \\
 A_{R_1} > A_{R_1} & \iff 
 \|\avector_{\bm \N_{R_1}} \|  >  \|\avector_{\bm \N_{R_2}} \|  & \qed \nonumber
\end{align}

\paragraph{Property III}
If $R_1$ outperforms $R_2$, for different rejected fractions $r_1 > r_2$, then $\|\avector_{\bm \N_{R_1}}\| = \|\avector_{\bm \N_{R_2}}\|$, and
\begin{align}
n Q_{R_1}(r_1) = \|\avector_{\bm \N_{R_1}} \| + \| 1- \avector_{\bm \R_{R_1}} \| = 
\| \avector_{\bm \N_{R_1}}  \| +  |{\bm \R_{R_1}} | - \|\avector_{\bm \R_{R_1}} \| = \nonumber \\
\|\avector_{\bm \N_{R_1}} \| + r_1 n - \|\avector \|  + |\avector_{\bm \N_{R_1}} \| > 
\|\avector_{\bm \N_{R_1}} \| + r_2 n - \|\avector \|  + |\avector_{\bm \N_{R_1}} \| = \nonumber \\
\|\avector_{\bm \N_{R_2}} \| + r_1 n - \|\avector \|  + |\avector_{\bm \N_{R_2}} \| = n Q_{R_2}(r_2) \qed \nonumber
\end{align}
If $R_1$ outperforms $R_2$, for different rejected fractions $r_1 < r_2$, then $\|1 - \avector_{\bm \N_{R_1}}\| = \|1 - \avector_{\bm \N_{R_2}}\|$, and
\begin{align}
n Q_{R_1}(r_1) = \|\avector_{\bm \N_{R_1}} \| + \|1 -  \avector_{\bm \R_{R_1}}  \| = 
\|\avector_{\bm \N_{R_1}} \| + |\bm \R_{R_1} |  - \|\avector_{\bm \R_{R_1}} \| = \nonumber \\
|\bm \N_{R_1} | + |\bm \R_{R_1} | - \|1- \avector_{\bm \N_{R_1}} \| - (\|\avector\| - \|\avector_{\bm \N_{R_1}} \|) = \nonumber \\
|\bm \N_{R_1} | + |\bm \R_{R_1} |- \|1- \avector_{\bm \N_{R_1}} \|  - \|\avector\| + |\bm \N_{R_1}| -  \|1 -\avector_{\bm \N_{R_1}} \| = \nonumber \\
n - A(0) +  |\bm \N_{R_1}|- 2 \|1- \avector_{\bm \N_{R_1}} \|  > 
n - A(0) +  |\bm \N_{R_2}|- 2 \|1- \avector_{\bm \N_{R_1}} \|  = \nonumber \\
n - A(0) +  |\bm \N_{R_2}|- 2 \|1- \avector_{\bm \N_{R_2}} \|  = n Q_{R_2}(r_2) \qed \nonumber
\end{align}

\paragraph{Property IV}
The classification quality achieves its unique maximum, $1$, if $\mathcal{A} = \mathcal{N}$ and $\mathcal{M} = \mathcal{R}$.
Conversely, it achieves its unique minimum, $0$, if$\mathcal{A} = \mathcal{R}$ and $\mathcal{M} = \mathcal{N}$.

\subsubsection{Rejection quality $\PH$}
\paragraph{Property I}
Let $B(r)$ denote the rejected accuracy $\|a_{\R}\| / |\R|$, we have
\begin{align}
A(r) (1-r) + B(r) r = A(0) \iff 
B(r) = \frac{A(0) - A(r) (1-r)  }{r}\nonumber
\end{align}
We can represent $\phi$ as 
\begin{align}
\phi = \frac{\|1- \avector_{\bm \R} \|}{\|\avector_{\bm \R} \|} \frac{\|\avector \|}{\|1- \avector \|} = \frac{1-B(r)}{B(r)} \frac{A(0)}{1-A(0)}  
\frac{r - A(0) + A(r) (1-r) }{A(0) - A(r) (1-r)}\frac{A(0)}{1 - A(0)} = \phi(r) \nonumber \qed
\end{align}

\paragraph{Property II}
For the same rejected fraction $r$, we have that if the rejection quality for $R_1$ is greater than the rejection quality for $R_2$, then
\begin{align}
\phi_{R_1}(r) > \phi_{R_2}(r) \iff
 \frac{\|1- \avector_{\bm \R_{R_1}} \|}{\|\avector_{\bm \R_{R_1}} \|} \frac{\|\avector \|}{\|1- \avector \|} >  \frac{\|1- \avector_{\bm \R_{R_2}} \|}{\|\avector_{\bm \R_{R_2}} \|} \frac{\|\avector \|}{\|1- \avector \|} 
\frac{\|1- \avector_{\bm \R_{R_1}} \|}{\|\avector_{\bm \R_{R_1}} \|}>  \frac{\|1- \avector_{\bm \R_{R_2}} \|}{\|\avector_{\bm \R_{R_2}} \|} \iff\nonumber \\
\frac{|\R_{R_1}| -\| \avector_{\bm \R_{R_1}} \|}{\|\avector_{\bm \R_{R_1}} \|}>  \frac{|\R_{R_2}|  - \| \avector_{\bm \R_{R_2}} \|}{\|\avector_{\bm \R_{R_2}} \|} 
\frac{|\R_{R_1}| -\| \avector_{\bm \R_{R_1}} \|}{\|\avector_{\bm \R_{R_1}} \|}>  \frac{|\R_{R_2}|  - \| \avector_{\bm \R_{R_2}} \|}{\|\avector_{\bm \R_{R_2}} \|} 
\frac{|\R_{R_1}| }{\|\avector_{\bm \R_{R_1}} \|} - 1>  \frac{|\R_{R_2}|  }{\|\avector_{\bm \R_{R_2}} \|} - 1 \iff\nonumber \\
\frac{\|\avector_{\bm \R_{R_1}} \|}{|\R_{R_1}| } <  \frac{\|\avector_{\bm \R_{R_2}} \|}{|\R_{R_1}|  } \iff
\|\avector\| - \|\avector_{\bm \N_{R_1}} \| < \|\avector\| -   \|\avector_{\bm \N_{R_2}} \| \iff   
\|\avector_{\bm \N_{R_1}} \| > \|\avector_{\bm \N_{R_2}} \| \qed \nonumber
\end{align}

\paragraph{Property III}
If $R_1$ outperforms $R_2$, for different rejected fractions $r_1 > r_2$, then $\|\avector_{\bm \N_{R_1}}\| = \|\avector_{\bm \N_{R_2}}\|$. As $\|\avector \| = \|\avector_\N\| + \|\avector_\R\|$ and $r_1 > r_2$, we have $\|\avector_{\bm \R_{R_1}}\| = \|\avector_{\bm \R_{R_2}} \|$ and $| \R_{R_1}| > |\R_{R_2}|$ respectively, leading to
\begin{align}
\phi_{R_1}(r_1) = \frac{\|1- \avector_{\bm \R_{R_1}} \|}{\|\avector_{\bm \R_{R_1}} \|} \frac{\|\avector \|}{\|1- \avector \|}  = \left( \frac{|\R_{R_1}|}{\|\avector_{\bm \R_{R_1}} \|}  - 1\right)\frac{\|\avector \|}{\|1- \avector \|} > 
> \left( \frac{|\R_{R_2}|}{\|\avector_{\bm \R_{R_1}} \|}  - 1\right) \frac{\|\avector \|}{\|1- \avector \|} =
\nonumber \\
\frac{\|1- \avector_{\bm \R_{R_2}} \|}{\|\avector_{\bm \R_{R_2}} \|} \frac{\|\avector \|}{\|1- \avector \|} =  \phi_{R_2}(r_2) \qed \nonumber
\end{align}
If $R_1$ outperforms $R_2$, for different rejected fractions $r_1 < r_2$,\emph{i.e.} $|\N_{R_1}| > |\N_{R_2}|$,  then $\|1 - \avector_{\bm \N_{R_1}}\| = \|1 - \avector_{\bm \N_{R_2}}\|$.
This means that $\|1 - \avector_{\bm \R_{R_1}}\| = \|1 - \avector_{\bm \R_{R_2}}\|$  and $|\R_1| < |\R_2|$, 
\begin{align}
\phi_{R_1}(r_1) = \frac{\|1- \avector_{\bm \R_{R_1}} \|}{\|\avector_{\bm \R_{R_1}} \|} \frac{\|\avector \|}{\|1- \avector \|}  = 
\frac{\|1- \avector_{\bm \R_{R_1}} \|}{| \R_{R_1}| - \|1- \avector_{\bm \R_{R_1}} \|} \frac{\|\avector \|}{\|1- \avector \|}  > 
\frac{\|1- \avector_{\bm \R_{R_1}} \|}{| \R_{R_2}| - \|1- \avector_{\bm \R_{R_1}} \|} \frac{\|\avector \|}{\|1- \avector \|}  = \nonumber \\
\frac{\|1- \avector_{\bm \R_{R_2}} \|}{| \R_{R_2}| - \|1- \avector_{\bm \R_{R_2}} \|} \frac{\|\avector \|}{\|1- \avector \|}  = \phi_{R_2}(r_2)\qed \nonumber
\end{align}

\paragraph{Property IV} 
The rejection quality achieves its maximum, $\infty$, when $\N = \mathcal{A}$ and $\mathcal{R}= \mathcal{M}$.
This maximum is not unique. Any selection of $\mathcal{R}$ such that $\mathcal{R} \subset \mathcal{M}$ results in maximum values of rejection quality.
Conversely, the rejection quality achieves its minimum, $0$, when $\mathcal{R} = \mathcal{A}$ and $\mathcal{N} = \mathcal{M}$.
This maximum is not unique. Any selection of $\mathcal{R}$ such that  $\mathcal{R} \subset \mathcal{A}$ results in minimum values of rejection quality.

\subsection{Comparing performance of classifiers with rejection}
Let us consider a classifier $C$ and two rejectors $R_1$ and $R_0$, with $r_1 > r_0$, and a cost function with a rejection-misclassification trade-off $\rho$.
Let $\beta$ be the relative optimality of the operating point of rejector $R_1$ at $r_1$ with respect to the reference operating point of $R_0$ at $r_0$.

From \eqref{eq:connect_loss}, and given the cost function with a rejection-misclassification trade-off $\rho$, we can relate outperformance, $\beta$ and $\rho$.

Rejector $R_1$ outperforms $R_0$ when
\begin{equation*}
\beta > 2 \rho - 1.
\end{equation*}

Rejector $R_0$ outperforms $R_1$ when
\begin{equation*}
\beta < 2 \rho - 1.
\end{equation*}

Rejector $R_0$ and $R_1$ are equivalent in terms of performance when
\begin{equation*}
\beta = 2 \rho - 1.
\end{equation*}

\subsubsection{Nonrejected accuracy $A$}
We can represent $A_{R_1}(r_1)$ as a function of $A_{R_0}(r_0)$ by noting that the best case scenario is 
\begin{equation*}
A_{R_1}(r_1) = A_{R_0}(r_0) \frac{1-r_0}{1-r_1} + \frac{r_1-r_0}{1-r_1},
\end{equation*}
corresponding to $\beta = 1$, and the worst case scenario is 
\begin{equation*}
A_{R_1}(r_1) = A_{R_0}(r_0) \frac{1-r_0}{1-r_1},
\end{equation*}
corresponding to $\beta = -1$.
This results in a representation of the nonrejected accuracy  $A_{R_1}(r_1)$ as
\begin{equation*}
A_{R_1}(r_1) = A_{R_0}(r_0) \frac{1-r_0}{1-r_1}  + \frac{\beta - 1}{2}\frac{r_1 - r_0}{1-r_1}.
\end{equation*}

Rejector $R_1$ outperforms $R_0$ when 
\begin{align}
A_{R_1}(r_1) = A_{R_0}(r_0) \frac{1-r_0}{1-r_1}  + \frac{\beta - 1}{2}\frac{r_1 - r_0}{1-r_1} 
>A_{R_0}(r_0) \frac{1-r_0}{1-r_1}  + (\rho - 1) \frac{r_1 - r_0}{1-r_1} .\nonumber
\end{align}

Rejector $R_0$ outperforms $R_1$ when 
\begin{align}
A_{R_1}(r_1) = A_{R_0}(r_0) \frac{1-r_0}{1-r_1}  + \frac{\beta - 1}{2}\frac{r_1 - r_0}{1-r_1}
<A_{R_0}(r_0) \frac{1-r_0}{1-r_1}  + (\rho - 1) \frac{r_1 - r_0}{1-r_1} .\nonumber
\end{align}

Rejector $R_0$ and $R_1$ are equivalent in terms of performance when
\begin{align}
A_{R_1}(r_1) = A_{R_0}(r_0) \frac{1-r_0}{1-r_1}  + \frac{\beta - 1}{2}\frac{r_1 - r_0}{1-r_1}
=A_{R_0}(r_0) \frac{1-r_0}{1-r_1}  + (\rho - 1) \frac{r_1 - r_0}{1-r_1} .\nonumber
\end{align}
\qed

\subsubsection{Classification quality $Q$}
We can represent $Q_{R_1}(r_1)$ as a function of $Q_{R_0}(r_0)$ by noting that the best case scenario is
\begin{equation*}
Q_{R_1}(r_1) = Q_{R_0}(r_0) + (r_1 - r_0),
\end{equation*}
corresponding to $\beta = 1$, and the worst case scenario is 
\begin{equation*}
Q_{R_1}(r_1) = Q_{R_0}(r_0) - (r_1 - r_0),
\end{equation*}
corresponding to $\beta = -1$.
This results in a representation of the classification quality $Q(R_1)(r_1)$ as
\begin{equation*}
Q_{R_1}(r_1) = Q_{R_0}(r_0) + \beta (r_1 - r_0),
\end{equation*}

Rejector $R_1$ outperforms $R_0$ when 
\begin{equation*}
Q_{R_1}(r_1) = Q_{R_0}(r_0) + \beta (r_1 - r_0) >  Q_{R_0}(r_0) + (2 \rho - 1) (r_1 - r_0).
\end{equation*}

Rejector $R_0$ outperforms $R_1$ when 
\begin{equation*}
Q_{R_1}(r_1) = Q_{R_0}(r_0) + \beta (r_1 - r_0) <  Q_{R_0}(r_0) + (2 \rho - 1) (r_1 - r_0).
\end{equation*}

Rejector $R_0$ and $R_1$ are equivalent in terms of performance when
\begin{equation*}
Q_{R_1}(r_1) = Q_{R_0}(r_0) + \beta (r_1 - r_0) =  Q_{R_0}(r_0) + (2 \rho - 1) (r_1 - r_0).
\end{equation*}
\qed
\end{document}